%% file: main.tex
\documentclass{article}

\usepackage{microtype}
\usepackage{graphicx}
\usepackage{subcaption}
\usepackage{booktabs} %

\usepackage{hyperref}

\input{math_commands}

\usepackage[accepted]{icml2026}

\usepackage{amsmath}
\usepackage{amssymb}
\usepackage{mathtools}
\usepackage{amsthm}
\usepackage{enumitem}
\usepackage{subcaption}

\usepackage[capitalize,noabbrev]{cleveref}

\crefname{appendix}{Appendix}{Appendices}
\Crefname{appendix}{Appendix}{Appendices}

\theoremstyle{plain}
\newtheorem{theorem}{Theorem}[section]
\newtheorem{proposition}[theorem]{Proposition}
\newtheorem{lemma}[theorem]{Lemma}

\theoremstyle{definition}

\theoremstyle{remark}

\usepackage[textsize=tiny]{todonotes}
\usepackage{xcolor}

\icmltitlerunning{Sequential Kernel-based Conditional Independence Testing via Adaptive Betting}

\begin{document}

\twocolumn[
  \icmltitle{Sequential Kernel-based Conditional Independence Testing via Adaptive Betting}

  \icmlsetsymbol{equal}{*}

  \begin{icmlauthorlist}
    \icmlauthor{Zheng He}{ubc}
    \icmlauthor{Danica J. Sutherland}{ubc,amii}
  \end{icmlauthorlist}

  \icmlaffiliation{ubc}{Department of Computer Science, University of British Columbia, Vancouver, Canada}
  \icmlaffiliation{amii}{Alberta Machine Intelligence Institute, Edmonton, Canada}

  \icmlcorrespondingauthor{Zheng He}{zhhe@cs.ubc.ca}
  \icmlcorrespondingauthor{Danica J. Sutherland}{dsuth@cs.ubc.ca}

  \icmlkeywords{Machine Learning, ICML}

  \vskip 0.3in
]

\printAffiliationsAndNotice{}  %

\begin{abstract}
  Testing conditional independence is fundamental
  yet intrinsically difficult:
  without additional assumptions, Type I error control is impossible in general.
  The ``Model-X'' paradigm addresses this difficulty by assuming exact knowledge of a relevant conditional distribution.
  While small deviations from this assumption can sometimes be tolerated in classical one-shot testing, 
  existing sequential conditional independence tests typically require the Model-X conditional to be known exactly, 
  making them fragile when it must instead be estimated.
  We propose a new approach that is substantially more robust to such estimation error. 
  Our method applies testing-by-betting to an adaptively optimized Kernel Conditional Independence statistic, 
  together with a normalization scheme and a truncate-and-shift calibration strategy.
  These modifications greatly reduce Type I error inflation while preserving high power
  across high-dimensional synthetic benchmarks and real-world fairness tasks, outperforming existing sequential Model-X approaches.
  Code is available at \url{https://github.com/he-zh/SKCI}.
\end{abstract}

\section{Introduction}

Do customers pay more for car insurance based on their race, even after controlling for neighborhood risk factors  \citep{angwin2017minority}?
Can predictive models satisfy the widely used notion of \emph{equalized odds} fairness, thereby mitigating such disparities
\citep{hardt:eq-opp}?
Does a pedestrian detection model remain robust across environments and weather conditions,
even when the true distribution of pedestrians shifts \citep{jiang-veitch}?
Are the shapes of a patient's cancer cells associated with the dosage of a particular medication, after controlling for demographics and disease progression \citep{zhao:distance-dists}?
These, and many more, fundamental questions can all be formulated as
problems of measuring or testing \emph{conditional independence} (CI):
given sample triples $(A, B, C)$,
we seek to reject the null hypothesis that $A \ind B \mid C$.

Classical approaches to null hypothesis testing,
however, have come under increasing criticism.
Widespread concerns about the ``reproducibility crisis'' in empirical science 
have highlighted the fragility of classical p-value-based inference under flexible data collection, multiple testing, and optional stopping
\citep[e.g.][]{forking-paths}.
In complex scientific and machine learning workflows, 
where data often arrive sequentially and analyses may adapt over time, 
more flexible tools are needed.
The theory of anytime-valid testing based on e-values and testing by betting \citep[see, e.g.][]{ramdas2025hypothesis}
provides a principled alternative
while remaining  compatible with likelihood-based reasoning and machine-learning-based estimators.

In this paper, we focus on sequential testing with an i.i.d. data stream, where obbservations arrive sequentially and the test may be
stopped at a data-dependent time.  Prior anytime-valid CI tests
\citep{shaer2023model,pandeva2024evaluating,pandeva2024davt} operate in the
Model-X setting \citep{candes:model-x-knockoffs}, assuming that the conditional
law \(P_{A\mid C}\) is known exactly.  This assumption is powerful because it
avoids the usual impossibility results for CI testing: without further
assumptions, no test can achieve nontrivial power
in either the fixed-sample \citep{shah2020hardness} or online setting \citep{waudby-smith:robbins-siegmund}.
We discuss these hardness results in \cref{sec:calibration-hard}.

Exact knowledge of \(P_{A\mid C}\), however, is often unrealistic -- especially in online settings.
We therefore propose a sequential test for CI
that works well when this conditional law is known,
but is designed to remain effective when the relevant
conditional distributions must be estimated online.
Since exact false-rejection control is not possible in full generality without additional assumptions,
we provide theoretical and empirical evidence that the method can
retain good Type~I error control and reasonable power in this estimated-conditional
regime.

In developing this test,
we also formalize a principle for choosing test
statistics that accumulate evidence quickly under weak signal. 
This principle is not specific to CI testing and may be
useful more broadly for testing by betting.

\section{Background}
\label{sec:background}
We begin by establishing the formal framework for sequential hypothesis testing through the betting paradigm, which reformulates statistical testing as an online game.

\subsection{Sequential Hypothesis Testing via Betting}
Consider a stream of i.i.d.~observations $(Z_t)_{t \geq 1}$,
where each $Z_t \in \Z$ arrives sequentially. 
Our goal is to test a null hypothesis $\Hzero$ regarding the distribution of this sequence.
Unlike fixed-sample tests,
a sequential test may reject at any data-dependent time.  Its Type I error
therefore measures the probability of ever rejecting \(\Hzero\) under the null.

The framework of \textit{testing by betting} \citep[overviewed by][]{ramdas2025hypothesis} reinterprets hypothesis testing as a game where a player attempts to accumulate wealth by betting against the null. The game is constructed such that no strategy can systematically profit if $\Hzero$ is true; conversely, a successful betting strategy under the alternative $\Hone$ yields capital growth, making the accumulated wealth a direct and intuitive measure of evidence against the null.

A player starts with an initial wealth $W_0 = 1$. 
At each round $t$, before observing \(Z_t\), 
the player chooses a \textit{payoff function} $f_t: \Z \to [-1, \infty)$ 
designed to capture the discrepancy between $\Hzero$ and $\Hone$, 
and bets a fraction $\lambda_t \in [0, 1]$ of their wealth; 
both $f_t$ and $\lambda_t$ should be 
\(\F_{t-1}\)-measurable, where
\(\F_{t-1}=\sigma(Z_1,\ldots,Z_{t-1})\) is the sigma-algebra generated by the observations available before round \(t\). 
Upon observing $Z_t$, their wealth becomes
\[
    W_t = W_{t-1}(1 + \lambda_t f_t(Z_t)).
\]
The test will be valid if
$\E_{\Hzero}[f_t(Z_t) \mid \F_{t-1}] \leq 0$,
as this makes the wealth process $(W_t)_{t \geq 0}$ a non-negative \textit{supermartingale} under $\Hzero$:
$\E_{\Hzero}[W_t \mid \F_{t-1}] \leq W_{t-1}$.
This allows for the application of Ville's inequality \yrcite{ville1939},
showing that for any $\alpha \in (0, 1)$,
\[
    \Pr\nolimits_\Hzero(\exists t \geq 1: W_t \geq 1/\alpha) \leq \alpha \E_\Hzero [W_0] = \alpha.
\]
Thus we construct an \emph{anytime valid} test of level $\alpha$
by rejecting the null
at the stopping time $\tau = \inf \{t \geq 1 : W_t \geq 1/\alpha\}$.
Among valid tests, we would like ones which choose $f_t$ and $\lambda_t$ so that $W_t$ grows rapidly under the alternative $\Hone$,
e.g.\ by maximizing the e-power $\E_{\Hone} [\log W_t]$.

\subsection{Designing Payoff Functions}

In most existing literature, $f_t$ is constructed by contrasting a statistic $g_t(\cdot)$ evaluated on the observed data $Z_t$ against the same statistic evaluated on a null-calibrated sample $\tZ_t$,
\[
f_t(Z_t) = g_t(Z_t) - g_t(\tZ_t),
\]
where $\tZ_t$ is constructed from $Z_t$ in a way such that $\E[g_t(\tZ_t)\mid \F_{t-1}] = \E_\Hzero[g_t(Z_t)\mid \F_{t-1}]$.

For instance, in \emph{two-sample testing},
$Z_t = (A_t, B_t)$ for stationary and independent $A$ and $B$;
we wish to test the null hypothesis that the distribution of $A$ equals that of $B$.
Then we might use $\tZ_t = (B_t, A_t)$,
since under $\Hzero$ the two distributions are the same \citep{shekhar2023nonparametric}.

We would like to select $g$ to distinguish the distributions as confidently as possible under $\Hone$,
e.g.\ by choosing a function class $\G$ and attempting to pick
\begin{equation}
\E[f^*(Z_t, \tZ_t)] = \sup_{g \in \G} \left( \E[g(Z_t)] - \E[g(\tZ_t)] \right).
\label{eq:optimal-power}
\end{equation}
If $-g \in \G$ for all $g \in \G$, the right-hand side is an integral probability metric \citep{mueller:ipm}, a common way to quantify distributional discrepancies.

\paragraph{Current Paradigms and Limitations}
\citet{shekhar2023nonparametric} and \citet{podkopaev2023sequential} restrict $\mathcal{G}$ to the unit ball of a Reproducing Kernel Hilbert Space (RKHS), where \eqref{eq:optimal-power} becomes the Maximum Mean Discrepancy (MMD; \citealp{gretton:mmd-jmlr}) between the distributions of $Z_t$ and $\tZ_t$.
\citet{shaer2023model} use the squared prediction error of a model trained on
past data as one example of \(g\). DAVT \citep{pandeva2024davt} instead chooses a neural network class for $\G$.

In all of these cases, the power depends on both the underlying distributional discrepancy and the quality of the chosen statistic $g$.
If the discrepancy is small, even a well-chosen statistic leads to slow wealth growth.
This problem cannot be dodged by arbitrarily scaling
\(\mathcal G\), since level control via Ville's inequality requires a nonnegative wealth
process and hence \(f_t(Z_t, \tZ_t)\ge -1\) almost surely.

\subsection{Challenges in Null Calibration for CI Testing}\label{sec:calibration-hard}
For two-sample testing or unconditional independence testing ($A \ind B$),
exact null-calibrated samples $\tZ$ can be easily obtained via permutation of the samples.
For conditional independence (CI) testing, where we ask if $A \ind B \mid C$,
generating $\tZ$ is substantially more difficult.
This is especially true if $C$ is continuous,
in which case we would typically observe only a single $(A, B)$ pair for any given value of $C$.

\citet{shah2020hardness} prove that no test with exact Type~I error control 
for all continuous null distributions can have
nontrivial power.
For any testing procedure, any sample size,
and any Lebesgue-continuous joint distribution of $(A, B, C)$ with $A \nind B \mid C$,
there exists another continuous joint distribution $(A', B', C')$ satisfying $A' \ind B' \mid C'$
that is indistinguishable to the test at this number of samples,
even if all distributions are restricted to have bounded support.
(The dependence can be hidden in ``hard-to-find'' features of the
conditioning variable $C$, such as lower-order bits.)

\citet{waudby-smith:robbins-siegmund} extend this hardness phenomenon
to sequential testing, showing that anytime-valid CI tests cannot guarantee
nontrivial power uniformly over all alternatives, even if the test is forced to
wait for an arbitrarily large number of observations before stopping.\footnote{%
It is worth briefly noting that these hardness results do not
necessarily rule out \emph{non-uniform} asymptotic control.
\citet{gyorfi-walk} claim their test has this property,
but their proof is incorrect, as pointed out by \citet[Section 1.1]{minimax-ci}
-- although note that \citeauthor{minimax-ci}'s statement that the result ``would otherwise seem to contradict'' \citet{shah2020hardness} is itself incorrect.
Whether there exists a procedure with such control is currently an open question;
\citet{boeken2026topological} show a partial negative result.}

To avoid these impossibility results, much of the current CI testing literature
imposes additional structure.
The most common example is the \emph{Model-X assumption}, which assumes the conditional distribution $P_{A \mid C}$ 
is known \cite{candes:model-x-knockoffs,shaer2023model,pandeva2024davt,grunwald2024anytime}.
Under this assumption, one can sample $\tA \sim P_{A \mid C}$
and form $\tZ=(\tA, B, C)$, which follows the null distribution by construction.
Other methods instead assume that \(P_{A,C}\)  is invariant across the null and alternative \cite{pandeva2024evaluating}. 
With such additional structure, exactly valid CI tests can be constructed.

In practice, we rarely have \emph{perfect} knowledge of the distribution of $A \mid C$.
One might instead estimate this conditional distribution from an auxiliary sample of
\((A,C)\) pairs, so that \(\tZ\) only
approximately follows the null. 
A sufficiently powerful test -- whether because it has many test observations
or because it identifies a strong distinguishing statistic \(g\) -- may 
then detect the mismatch between \(Z\) and the incorrectly
generated \(\tZ\), and reject even when \(A\ind B\mid C\).
\citet{pogodin2024splitkci} and \citet{he2025hardness} study related calibration failures
for conditional testing in the batch setting.
The online setting is even more demanding, since validity must hold across an unbounded number of stopping
opportunities, and approximation error in \(\tZ\) must remain small enough
to not become detectable as arbitrarily more observations arrive.
Moreover, the auxiliary-data regime itself can be
unrealistic in applications where \((A,B,C)\) triples arrive only online.

In this work, we avoid constructing an explicit null-calibrated sample \(\tZ\). 
Instead, we use payoffs of the form
\(
    f_t(Z_t) = g_t(Z_t;\gamma_t),
\)
where $\gamma_t$ is a data-dependent shift chosen to make the resulting wealth process approximately 
a supermartingale under the null.

\section{Methodology}
\label{sec:method}
In this section, we present our sequential testing framework for conditional independence. 
We construct betting payoffs that are designed to promote rapid wealth growth under alternatives while mitigating Type~I error inflation beyond the exact Model-X setting.

Consider a stream of i.i.d.\ observations $Z_t = (A_t, B_t, C_t)$, each in a space $\Z = \A \times \B \times \C$, arriving in batches. 
To ensure that the payoff function and the betting fraction are predictable ($\F_{t-1}$-measurable), we adopt a sequential data partitioning scheme that separates model estimation, calibration, and testing. At each round t, the observed data are divided into three disjoint subsets:
\begin{itemize}[topsep=0pt,leftmargin=*]
\item \textbf{Training set} $(\Xtr_{t-1})$:
Used to estimate the data-dependent quantities appearing in the test statistic.
This set grows monotonically over time.

\item \textbf{Validation set} $(\Xval_{t-1})$:
A held-out buffer used to determine calibration quantities required by the test.

\item \textbf{Test batch} $(\Y_t)$:
Fresh observations used to update the wealth process $W_t$.
\end{itemize}

Let $\H_{t-1} = \Xtr_{t-1} \cup \Xval_{t-1}$
denote the historical data available prior to observing $\Y_t$, and define the
associated filtration $\F_{t-1} = \sigma(\H_{t-1})$.
We update the wealth according to
\begin{equation}
W_t = W_{t-1}(1 + \lambda_t V_t), \quad \lambda_t \in (0,1),
\end{equation}
where $V_t$ denotes the payoff used at round $t$, corresponding
to $f_t(\Y_t)$ in the notion of payoff functions from \cref{sec:background}.
If $\lambda_t$ is $\F_{t-1}$-measurable,
and $\E_{\Hzero}[ V_t \mid \F_{t-1} ] \le 0$,
then the wealth process is a supermartingale under $\Hzero$.

Before the next batch,
we update the data partitions as
\begin{align*}
\Xtr_{t} \gets
\Xtr_{t-1} \cup \Xval_{t-1}, \qquad
\Xval_{t} \gets \Y_{t-1}.
\end{align*}
The batch of data observed at timestep $t$ is used first as $\Y_t$,
then as $\Xval_{t+1}$,
and then for all $t' \ge t + 2$ is included in $\Xtr_{t'}$.

\paragraph{A Self-Normalized Payoff}
We will build our measure of discrepancy between  $\Hzero$ and $\Hone$
from a symmetric kernel\footnote{``Kernel'' is an extremely overloaded word in machine learning and statistics. We use it here in the sense of a $U$-statistic kernel, for which the only requirement is symmetry $h(z, z') = h(z', z)$. Our choice will also turn out to be the kernel of a reproducing kernel Hilbert space, but that is not necessary for our normalization.}
$h : \Z \times \Z \to \R$ designed to capture conditional independence.
Ideally, $h$ vanishes on average under $\Hzero$,
$\E_{\Hzero}[h(Z, Z') \mid Z] = 0$,
and takes large values under $\Hone$.

We then construct our ``raw'' betting score by comparing the new data to historical observations with $h$,
using a form that can be described as a cross U-statistic \citep{kim-ramdas:cross-u}.
Given a training history $\Xtr_{t-1} = \{x_i\}_{i=1}^n$ and a fresh batch
$\Y_t = \{y_j\}_{j=1}^b$, let
\begin{align*}
U_{n,b}(\Xtr_{t-1}, \Y_t)
\defeq \frac{1}{nb} \sum_{i=1}^n \sum_{j=1}^b h(x_i, y_j).
\end{align*}
The ``cross'' structure is particularly convenient for taking conditional expectations.
Conditional on the past data $\F_{t-1}$, 
the reference points $\Xtr_{t-1}=\{x_i\}_{i=1}^n$ are fixed, 
while the new observation $\Y_t$ is independent of $\F_{t-1}$. Hence
\[
\E\left[
U_{n,b}(\Xtr_{t-1}, \Y_t)
\mid
\F_{t-1}
\right]
=
\frac{1}{n} \sum_{i=1}^n
\E[h(x_i, Z) \mid F_{t-1}]
,\]
where $Z$ denotes an independent draw from the data-generating distribution.
Under $\Hzero$, this conditional mean is ideally zero regardless of $x_i$.

When the discrepancy between $\Hzero$ and $\Hone$ is weak (or $h$ is poor),
however,
the magnitude of $U_{n,b}$ can be small,
leading to slow wealth growth under the alternative.
To adapt to the unknown scale of the kernel interaction, we introduce the
statistic $S_n$, computed entirely from historical data:
\[
S_n(\Xtr_{t-1})
\defeq
\frac{1}{n^2} \sum_{i=1}^n \sum_{j=1}^n h(x_i, x_j).
\]
The V-statistic $S_n$ is $\F_{t-1}$-measurable,
and so we can easily incorporate it into our bets.
We use
\begin{equation}
\Vr_t\defeq
\frac{U_{n,b}(\Xtr_{t-1}, \Y_t)}
     {S_n(\Xtr_{t-1}) + \varepsilon},
\quad \varepsilon > 0
\label{eq:V_raw}
.\end{equation}
With large $n$ and $b$, both $U_{n,b}$ and $S_n$ converge to $\E h(X,Y)$.
Thus, for small $\varepsilon$, $\Vr_t \approx 1$ under alternative distributions, 
independently of the scale of $h$.
Moreover, if
\(
    \E_{\Hzero}[h(x_i,Z)\mid \F_{t-1}] = 0
\)
, then the numerator is conditionally mean zero under the null. 
Since the denominator is $\F_{t-1}$-measurable, 
the normalized $\Vr_t$ remains conditionally mean zero, 
and the resulting wealth process is a martingale under the null.
The regularization parameter $\varepsilon$ prevents instability when the denominator is close to zero, but should be much smaller than $\E_\Hone h(X, Y)$ to preserve power.

\paragraph{The KCI Operator}
We now specify our choice of $h$ via the
Kernel-based Conditional Independence (KCI) framework \cite{zhang2012kernel}, which
provides a principled RKHS representation of conditional independence.

Map $\A$, $\B$, and $\C$ into reproducing kernel Hilbert spaces
$\H_A$, $\H_B$, and $\H_C$ with feature maps
$\phia$, $\phib$, and $\phic$, respectively.
Throughout, we assume that all RKHSs considered are separable and that the
feature maps are measurable.
The conditional mean embeddings
$\muac(c) \defeq \E[\phia(A) \mid C = c]$ and
$\mubc(c) \defeq \E[\phib(B) \mid C = c]$
represent the components of $A$ and $B$ explained by $C$.

Following \citet{he2025hardness}, we construct the KCI operator
\(\psi(Z)\), which captures the residual interaction between $A$ and $B$ after
conditioning on $C$:
\begin{align*}
\psi(z)\!=\!\bigl(\phia(a)\! - \!\muac(c)\bigr) \otimes \bigl(\phib(b) \! - \! \mubc(c)\bigr) \!\otimes \!\phic(c).
\end{align*}
Intuitively, $\psi(Z)$ encodes the dependence between the residuals of $A$ and $B$
after removing the effect of $C$, weighted by the representation of $C$.
Under the null,
the residualized features are conditionally
uncorrelated given \(C\), and hence
\(
    \E_{\Hzero}[\psi(Z)] = 0 .
\)
Under suitable universality conditions on the kernels,
\(\E[\psi(Z)]\neq 0\) for any violation of conditional independence.
We therefore define the kernel by
\begin{equation}
h(Z, Z') \defeq \langle \psi(Z), \psi(Z') \rangle.
\end{equation}
Then, for any fixed $Z$,
\[
\E_{\Hzero}[h(Z, Z') \mid Z]
=
\langle \psi(Z), \E_{\Hzero}[\psi(Z')] \rangle
= 0 .
\]
So far, we have treated the component kernels as fixed in advance and the
conditional mean embeddings as known population quantities.  
In a practical sequential setting, however, the
kernel may be updated as more data become available.
Outside the Model-X setting, the conditional mean embeddings \(\muac\) and \(\mubc\) are
unknown and must be estimated from historical data \(\Xtr\).

At each round $t$, we construct the kernel $h^{(t)}$ using conditional mean
embeddings estimated from the historical training data $\Xtr_{t-1}$.
Consequently, \(h^{(t)}\) is allowed to depend on all past
information, but remains \(\F_{t-1}\)-measurable.
The round-$t$ payoffs $\Vr$, and $V_t$ to be described next, are then computed from
$(\H_{t-1}, \Y_t)$ using this kernel.

\paragraph{The Shift-and-Truncate Mechanism}
To apply Ville's inequality,
the betting payoff must satisfy two properties: $V_t \ge -1$ almost surely 
and $\E_{\Hzero}[V_t \mid \F_{t-1}] \le 0$ for every round $t$.
Although the kernel \(h\) is bounded, the raw
self-normalized score \(\Vr_t\) defined in \eqref{eq:V_raw} need not be bounded
below by \(-1\).  In particular, fluctuations in the normalization term can
produce large negative values, which would make the corresponding wealth update
invalid.

We therefore define the payoff by applying a predictable shift $\gamma_t$
followed by a one-sided truncation:
\begin{equation} 
V_t
\defeq
\max\{ \Vr_t - \gamma_t,\,-1 \}.
\label{eq:truncate_shift}
\end{equation}
The truncation enforces \(V_t \ge -1\), ensuring that the wealth process remains
nonnegative.  However, truncation alone can increase the conditional mean of the
payoff.  To preserve null calibration, we choose the smallest nonnegative shift
\(\gamma_t\) such that
\begin{equation}\!\!\!
\gamma_t
\defeq
\min_{\gamma \ge 0}\Bigl\{
\gamma :
\E_{\Hzero}\!\left[
\max\{ \Vr_t \!-\! \gamma,\!-1 \}
\!\mid\!
\F_{t-1}
\right]
\le 0
\Bigr\}.
\label{eq:gamma_def}
\end{equation}
This defines an ideal predictable shift whenever the corresponding null expectations
can be evaluated exactly.

Under \(\Hone\), when the kernel is aligned with the conditional dependence
signal, the raw score \(\Vr_t\) has positive conditional mean and is less likely
to fall far below zero.  In this regime, the truncation is rarely active and the
required shift \(\gamma_t\) is relatively small.

Thus, the mechanism is
designed to retain the power benefits of the self-normalized payoff.  Exact
anytime-valid Type~I error control, however, is guaranteed for the ideal version
in which the conditional null expectation in \eqref{eq:gamma_def} is computed
exactly. In practice, we approximate the expectation using the scheme described below.

\paragraph{Gaussian Approximation for Shift Estimation}
The ideal shift in \eqref{eq:gamma_def} depends on the
conditional null distribution of the raw score \(\Vr_t\).  Since this
distribution is not available in general, we approximate it by a Gaussian
law.

This approximation is motivated by the structure of \(\Vr_t\).    
For the round $t$ test batch $\Y_t=\{y_j\}_{j=1}^b$, let
$\Xtr_{t-1}=\{x_i\}_{i=1}^n$ denote the historical training samples. Define 
\[
g^{(t)}(y)
\defeq
\frac{1}{n\bigl(S_n(\Xtr_{t-1})+\varepsilon\bigr)}
\sum_{i=1}^n h^{(t)}(x_i,y).
\]
Then the raw score on the test batch can be written as
\[
  \Vr_t = \frac{1}{b}\sum_{j=1}^b g^{(t)}(y_j).
\]
Conditional on the past information \(\F_{t-1}\), 
each of
the kernel \(h^{(t)}\), the training points
\(\{x_i\}_{i=1}^n\),  and the normalization
\(S_n(\Xtr_{t-1})+\varepsilon\) are fixed.
Hence, \(g^{(t)}\) is \(\F_{t-1}\)-measurable, and \(\Vr_t\) is a normalized average of
test-sample contributions.  As the test samples in the batch are independent,
if \(b\) is sufficiently large,
by the central limit theorem
\[
    \operatorname{Law}\bigl( \Vr_t \mid \F_{t-1} \bigr) \approx \N(\mu_t,\sigma_t^2).
\]
We then approximate the shift by replacing the unknown conditional null law of
\(\Vr_t\) in \eqref{eq:gamma_def} with this Gaussian law.
Under this approximation,
the conditional null expectation in
\eqref{eq:gamma_def} can be evaluated in closed form.
For a random variable
\(V\sim \N(\mu,\sigma^2)\), define
\[
    f(\gamma;\mu,\sigma)
    \defeq
    \E_{V\sim \N(\mu,\sigma^2)}\bigl[\max\{V-\gamma,-1\}\bigr].
\]
Defining
\(
    \xi \defeq \frac{\gamma-\mu-1}{\sigma}
,\)
we can directly calculate that
\[
    f(\gamma;\mu,\sigma)
    =
    \sigma
    \left[
        \phi(\xi)
        -
        \xi\Phi(-\xi)
    \right]
    -1,
\]
where \(\phi\)
denotes the density of a standard normal
and \(\Phi\) its cumulative distribution function.
The Gaussian plug-in shift is then chosen as the smallest nonnegative value of
\(\gamma\) for which \(f(\gamma;\mu_t,\sigma_t)\le 0\).

It remains to specify the plug-in parameters \(\mu_t\) and \(\sigma_t^2\).  If
the conditional mean embeddings used in \(h^{(t)}\) are exact, then the raw
score is centered under \(\Hzero\).  With estimated conditional mean
embeddings, this centering is only approximate,
but it is difficult to accurately estimate the true null mean;
we therefore use the approximation
\(
    \widehat \mu_t = 0
\).

Under the approximation \(\widehat \mu_t = 0\),  the conditional variance of
of the normalized raw score satisfies
\begin{align*}
    \Var_{\Hzero}[\Vr_t\mid \F_{t-1}]
    &= \frac{1}{b} \Var_{\Hzero}\left[ g^{(t)}(Y) \mid \F_{t-1}\right]
    \\
    &\approx \frac{1}{b} \E_{\Hzero}[(g^{(t)}(Y))^2\mid \F_{t-1}],
\end{align*}
where \(Y\) denotes an independent null sample.

During shift estimation, 
after these quantities have been fixed, 
we estimate null scale using the
validation set \(\Xval_{t-1}\)
and use this estimate to calibrate the final shift \(\hgamma_t\).

Let \(\{v_j\}_{j=1}^b \subseteq \Xval_{t-1}\) denote the validation points.  We define
\begin{equation}
    \widehat\sigma_t^2
    \defeq
    \frac{1}{b^2}
    \sum_{j=1}^b \bigl(g^{(t)}(v_j)\bigr)^2 .
    \label{eq:sigma-estimate}
\end{equation}
This separation reduces the bias that would arise from optimizing the score and estimating its
null scale on the same training samples.
The practical shift is then chosen as
\begin{equation}
    \widehat\gamma_t
    \defeq
    \min\Bigl\{
        \gamma \ge 0 :
        f(\gamma;0,\widehat\sigma_t) \le 0
    \Bigr\}.
    \label{eq:gamma-gaussian}
\end{equation}
Since \(f(\gamma;0,\widehat\sigma_t)\) is monotone nonincreasing in \(\gamma\),
\(\widehat\gamma_t\) can be found efficiently by binary search.  Restricting to
\(\gamma\ge 0\) is not strictly necessary,
but makes the correction more conservative.

\begin{algorithm}[t]
\caption{Sequential KCI Testing via Betting}
\label{alg:skb}
\begin{algorithmic}[1]
\STATE \textbf{Input:} Initial training set $\Xtr_{0}$, initial validation set $\Xval_0$,
stream of batches $\{\mathcal{Y}_t\}_{t=1}^\infty$ of size $b$, threshold $1/\alpha$, regularization $\varepsilon$.
\STATE \textbf{Initialize:} $W_0 \gets 1$, $\eta_0 \gets 0$, $\hgamma_0 \gets 0$, kernel $h^{(0)}$.
\FOR{$t = 1, 2, \dots$}

    \STATE \COMMENT{\textit{Phase 1: Conditional Means Estimation}}
    \STATE Select kernel hyperparameters and fit
$\muac^{(t)}$ and $\mubc^{(t)}$ via kernel ridge
regression on $\Xtr_{t-1}$.
    \STATE \COMMENT{\textit{Phase 2: Strategy Optimization \& Shift Estimation}}
    \STATE Set
    $\eta_t^{(0)} \gets \eta_{t-1}$,
    $h^{(t,0)} \gets h^{(t-1)}$,
    $\hgamma_t^{(0)} \gets \hgamma_{t-1}$.
    \FOR{$s = 1, \dots, S$} 
        \STATE Take a gradient step to update $\eta_t^{(s)}$ and $h^{(t,s)}$ (in the parameters of $k_C$ only) to maximize
        \[
        \displaystyle\sum_{i=1}^t
        \log\!\left(
        1 + \bm{\sigma}(\eta_t^{(s)})\,
        \!\max\{\tV_i^{(t,s)} - \hgamma_t^{(s-1)}, -1\}\!
        \right)
        \]
        where historical payoff proxies $\{\tV_i^{(t,s)}\}_{i=1}^t$ are based on the current kernel.
        \STATE Select the shift $\hgamma_t^{(s)}$ for the new kernel via binary search 
        using $\mathcal{X}_{\mathrm{tr}}^{(t-1)}$ and $\mathcal{X}_{\mathrm{val}}^{(t-1)}$.
    \ENDFOR
    \STATE Set $\eta_t \gets \eta_t^{(S)}$, $h^{(t)} \gets h^{(t,S)}$, $\hgamma_t \gets \hgamma_t^{(S)}$.

    \STATE \COMMENT{\textit{Phase 3: Wealth Update}}
    \STATE Receive new test batch $\mathcal{Y}_t$.
    \STATE Compute $\Vr_t \gets \frac{U_{n,b}^{(t)}}{S_n^{(t)} + \varepsilon}$ using $h^{(t)}$.
    \STATE Compute final payoff $V_t \gets \max\{\Vr_t - \hgamma_t, -1\}$.
    \STATE Update wealth $W_t \gets W_{t-1}\bigl(1 + \bm{\sigma}(\eta_t) V_t\bigr)$.

    \IF{$W_t \ge 1/\alpha$}
        \STATE \textbf{Reject $\Hzero$ and terminate.}
    \ENDIF
    \STATE $\mathcal{X}_{\mathrm{tr}}^{(t)} \gets \mathcal{X}_{\mathrm{tr}}^{(t-1)} \cup \mathcal{X}_{\mathrm{val}}^{(t-1)},
    \quad \mathcal{X}_{\mathrm{val}}^{(t)} \gets \mathcal{Y}_t$
\ENDFOR
\end{algorithmic}
\end{algorithm}

\paragraph{Parameter Selection}
The kernel choices of KCI play two distinct roles in our procedure.  
First, the regression kernels used for conditional mean embedding estimation
should be chosen to obtain good CME estimates. 
The kernel on the conditioning variable \(C\)
also affects the sensitivity of the KCI statistic to conditional dependence; this is
particularly important in difficult CI testing problems where the signal can be
hidden by an inappropriate conditioning kernel \citep{he2025hardness}.

For the regression step, we follow \citet{pogodin2024splitkci}.  We use
fixed kernels for \(A\) and \(B\), with heuristic bandwidth choices, and select the regression kernels \(k_{C\to A}\) and
\(k_{C\to B}\) by minimizing leave-one-out prediction error.  These choices are
used to estimate the conditional mean embeddings.  Implementation details are
given in \cref{apx:implementation:kernels}.

We then choose the conditioning kernel \(k_C\) together with the betting
fraction \(\lambda_t\), since both affect the growth of the wealth process.  In
principle, the natural objective is the expected logarithmic wealth increment,
\[
    \argmax_{\lambda,k_C}
    \E_{\Hone}\!\left[\log(1+\lambda V_t)\right],
\]
which is the standard criterion for maximizing asymptotic wealth growth and is
closely connected to powerful e-value constructions
\citep[see, e.g.][Sections 3.3 and 7.8]{ramdas2025hypothesis}.  Since the alternative
distribution is unknown, we replace this population objective by an empirical
proxy computed from the historical data available before round \(t\).

At each round \(t\), we compute proxy payoffs computed with the current kernel \(h^{(t)}\).  
Since the observations are i.i.d., these proxy payoffs
provide a pre-round estimate of how the current kernel and betting fraction
would perform on fresh test batches.

To construct the proxy payoffs, we partition the \(n\) historical training
samples \(\Xtr_{t-1}\) into blocks of size \(b\), and write \(\mathcal I_i\)
for the indices in the \(i\)th block. 
Since data arrives in batches of size $b$, we have $t$ blocks in total.
For each block, we treat the samples in
\(\mathcal I_i\) as a pseudo-test batch and compare them with the remaining
historical samples.  

Specifically, we define
\begin{equation}
\tV_i^{(t)}
=
\frac{
\sum_{l \in \mathcal I_i}
\sum_{j \notin \mathcal I_i}
h^{(t)}(x_j, x_l)
}{
b (n-b) \left(S_n^{(t)}(\Xtr_{t-1}) + \varepsilon\right)
},
\quad
i \in \{1, \dots, t\}.
\end{equation}
This leave-block-out construction avoids self-interaction terms and gives an
empirical proxy for the current normalized payoff.  We then select \(k_C\) and
\(\lambda_t\) by maximizing the empirical log-wealth criterion over these proxy
payoffs.

To enforce the constraint $\lambda_t \in (0,1)$, we parameterize
$\lambda_t = \bm{\sigma}(\eta_t)$ using the logistic sigmoid function,
and jointly optimize it along with kernel parameters on
our block estimate of the expected log-wealth growth: 
\begin{equation}
\sum_{i=1}^t
\log\!\left(
1 + \bm{\sigma}(\eta_t)
\max\{\tV_i^{(t)} \! - \gamma^{(t)}, -1\}
\right).
\label{eq:wealth-obj}
\end{equation}
The overall framework for our sequential testing procedure,
which we call Sequential Kernel Conditional Independence (SKCI),
is summarized in
\cref{alg:skb}.

\section{Theoretical Analysis}
\label{sec:theory}

We analyze the statistical validity of the proposed sequential test under the
null hypothesis \(\Hzero\).  The main difficulty is that the wealth process
\(W_t\), constructed using the implemented Gaussian shift \(\hgamma_t\), is not
an exact null supermartingale.  We quantify the deviation from the ideal
supermartingale property through the conditional drift
\begin{equation} \label{eq:drift-def}
    \delta_t
    \defeq
    \E_{\Hzero}\!\left[
        \max\{\Vr_t-\hgamma_t,-1\}
        \,\middle|\,
        \F_{t-1}
    \right].
\end{equation}
If \(\delta_t\le 0\), the one-step wealth update is conservative under
\(\Hzero\), so we
derive an explicit deterministic upper bound on \(\delta_t\) and use it to
construct a compensated supermartingale.

The full proof is given in \cref{app:proof_type1}.  We summarize the argument
in three steps: first, we establish the sensitivity of the Gaussian
null-calibrating shift; second, we combine this sensitivity bound with a
finite-sample Gaussian approximation to obtain the one-step drift bound; third,
we convert the one-step drift bound into the Type~I error bound.

\subsection{Sensitivity of the Null-Calibrating Shift}

We define the Gaussian-calibrated shift using the Gaussian approximation to the
conditional null law of \(\Vr_t\).  For parameters \((\mu,\sigma)\), let
\[
    f(\gamma;\mu,\sigma)
    \defeq
    \E_{Z\sim \mathcal N(\mu,\sigma^2)}
    \bigl[\max\{Z-\gamma,-1\}\bigr].
\]
The Gaussian calibrating shift \(\gamma(\mu,\sigma)\) is defined by
\[
    f(\gamma(\mu,\sigma);\mu,\sigma)=0.
\]
In particular, \(\gamma(\bmu_t,\bsigma_t)\) denotes the Gaussian shift
corresponding to the limiting approximation
\(N(\bmu_t,\bsigma_t^2)\) of \(\Vr_t\mid\F_{t-1}\), while the implemented shift
\(\hgamma_t\) is selected using the centered Gaussian law
\(N(0,\hsigma_t^2)\).

The following lemma controls the effect of using the wrong Gaussian mean and
variance.

\begin{lemma}[Implicit Sensitivity]
\label{lem:sensitivity_main}
The map \((\mu,\sigma)\mapsto \gamma(\mu,\sigma)\) is continuously
differentiable.  Moreover, for any \((\bmu,\bsigma)\) and
\((\hmu,\hsigma)\),
\[
    \gamma(\bmu,\bsigma)-\gamma(\hmu,\hsigma)
    \le
    \bmu-\hmu+\Lambda(\bsigma)(\bsigma-\hsigma),
\]
where
\(
    \Lambda(\sigma)
    \defeq
    \frac{\phi(\xi_\sigma)}{\Phi(-\xi_\sigma)},
\)
and \(\xi_\sigma\) is the unique solution of
\(
    \sigma\{\phi(\xi)-\xi\Phi(-\xi)\}=1.
\)
Here \(\phi\) and \(\Phi\) denote the standard normal density and distribution
function.
\end{lemma}

This lemma allows us to convert mean and variance mismatch in the Gaussian
calibration law into an upper bound on the excess expected payoff.

\subsection{Drift Decomposition}

The drift has two sources: a finite-block Gaussian approximation gap for
\(\Vr_t\mid\F_{t-1}\), and a Gaussian calibration mismatch from using
\(\hgamma_t\), calibrated under \(N(0,\hsigma_t^2)\) instead  \(N(\bmu_t,\bsigma_t^2)\).  Let \(\detact\) and \(\detbct\)
denote the conditional mean embedding regression errors at step \(t\).

\begin{theorem}[Drift Upper Bound]
\label{lem:drift_bound_main}
Under \(\Hzero\), assume the kernel is bounded such that $\sup_{x} {h^{(t)}(x, x)} \le \kappa$,
and \(\sup_c\norm{\phi_C(c)}\le 1\).
Assume that the conditional absolute third central moment of \(\frac{1}{n}\sum_{i=1}^n h^{(t)}(x_i,Y)\) is uniformly bounded
and that its conditional variance is uniformly nondegenerate,
so that its Berry–Esseen moment ratio is at most $\rho$.
Then for each \(t\le T\),
the drift $\delta_t$ of \eqref{eq:drift-def}
satisfies
\[
    \delta_t \le
    U_t
    \defeq
     \frac{C_1\rho}{b\, \varepsilon}
    +
    \frac{\sqrt{\kappa}}{\varepsilon}
    \norm[\big]{\detact}\norm[\big]{\detbct}
    +
    \frac{2C_2 \kappa^2}{b\, \varepsilon^2}.
\]
\end{theorem}
The first term comes from the finite-block Gaussian approximation gap, which we control using a Wasserstein Berry--Esseen bound for \(\Vr_t\); here
\(C_1>0\) is a universal constant.
The second and third terms arise from the Gaussian calibration mismatch
controlled by \cref{lem:sensitivity_main}.  The second term captures the effect
of the nonzero conditional mean \(\bmu_t\), which is induced by conditional mean
embedding estimation error and vanishes when the residualizers recover the
population conditional mean embeddings.  The third term controls the mismatch
between the Gaussian variance parameter \(\bsigma_t^2\) and the implemented
variance estimate \(\hsigma_t^2\); here \(C_2>0\) is another universal constant.

The ablations in \cref{fig:skci-ablation} are consistent with this drift upper bound:
larger batch sizes \(b\) and regularization parameters \(\varepsilon\) tend to
reduce empirical null rejection rate, as suggested by the corresponding terms in
\(U_t\).

\subsection{Finite-Sample Type I Error Inflation}

The following result translates a
one-step drift bound, such as \cref{lem:drift_bound_main},
into a Type~I error bound.
Since
\(
    \E_{\Hzero}[V_t\mid\F_{t-1}]
    =
    \delta_t
    \le
    U_t,
\)
the corrected process
\(
    \widetilde W_t
    \defeq
    \frac{W_t}{\prod_{i=1}^t(1+\lambda_i U_i)}
\)
is a nonnegative supermartingale under \(\Hzero\),
and Ville's inequality
yields the following result.

\begin{proposition}[Finite-Sample Type~I Error]
\label{thm:finite_sample_type1_main}
Let \(\alpha\in(0,1)\).  Under \(\Hzero\), suppose that for all \(t\),
\(
    \delta_t\le U_t.
\)
Then a test rejecting when
\( \displaystyle
    W_t
    \ge
    \frac1\alpha
    \prod_{i=1}^t(1+\lambda_i U_i)
\)
has anytime control of its Type~I error at level $\alpha$.
This also implies that the original test satisfies
\[
    \Pr_{\Hzero}\!\left(
        \exists t\le T:
        W_t\ge \frac1\alpha
    \right)
    \le
    \alpha
    \exp\left(
        \sum_{t=1}^T \lambda_t U_t
    \right).
\]
\end{proposition}

\section{Experiments}
\label{sec:experiments}
\begin{figure}[t]
    \centering
    \includegraphics[width=\columnwidth]{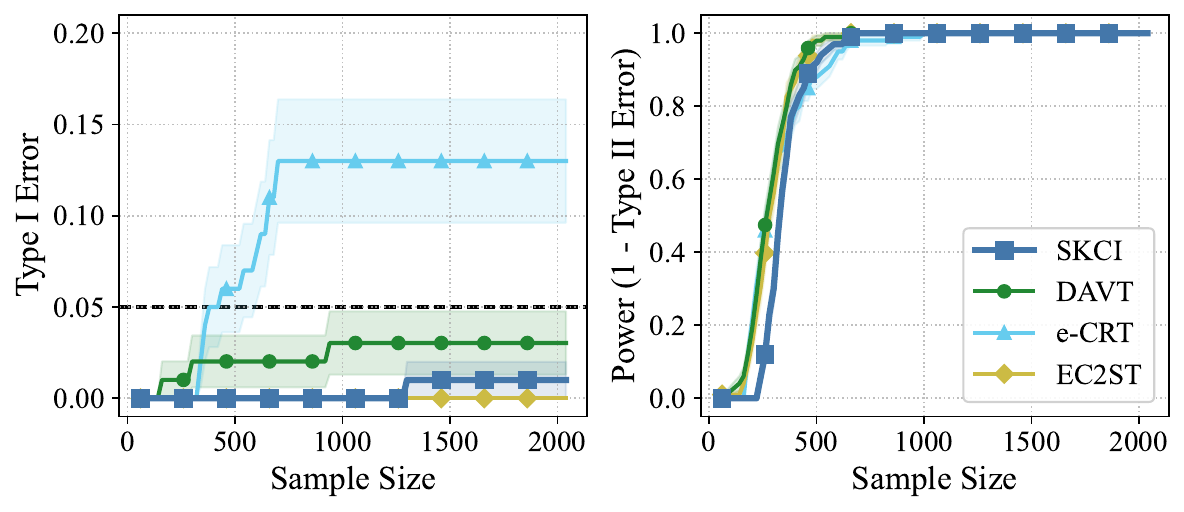}
    \caption{Linearly dependent Gaussian data in online mode.}
    \label{fig:gaussian-online}
\end{figure}

We evaluate the SKCI betting framework across synthetic and real-world benchmarks designed to evaluate both anytime Type~I error control and power for difficult problems.

Across all experiments, we compare SKCI against e-CRT \citep{shaer2023model}, DAVT \citep{pandeva2024davt},
and a version of EC2ST \citep{pandeva2024evaluating} that distinguish true $(A, B, C)$ triples from $(\widetilde A,  B, C)$ knockoffs.
We use a common batch size $b=20$, and average
the results over 100 independent runs
unless otherwise specified. Shaded regions in all plots indicate the standard error.
We use the same regression architecture for baseline methods for fair comparison; details in \cref{apx:baselines}.

To isolate the effect of estimation quality, we report results in three regimes.
(i) In Oracle mode, the conditional distribution $P_{A\mid C}$ is known exactly; this is possible for synthetic data only.
(ii) In Pretrained mode, $P_{A\mid C}$ is estimated from a large offline dataset (we use 3,000 samples).
(iii) In Online mode, no prior side data is available, and conditional mean embeddings are updated sequentially as we see data.

Due to space constraints,
Oracle results, Pretrained results
and additional configurations and ablation studies
are deferred to \cref{apx:exp-results}.

\begin{figure}[ht]
\captionsetup[sub]{skip=0pt}
\captionsetup{skip=-4pt}
\begin{subfigure}[ht]{\linewidth}
  \begin{center}
    \centerline{\includegraphics[width=\columnwidth]{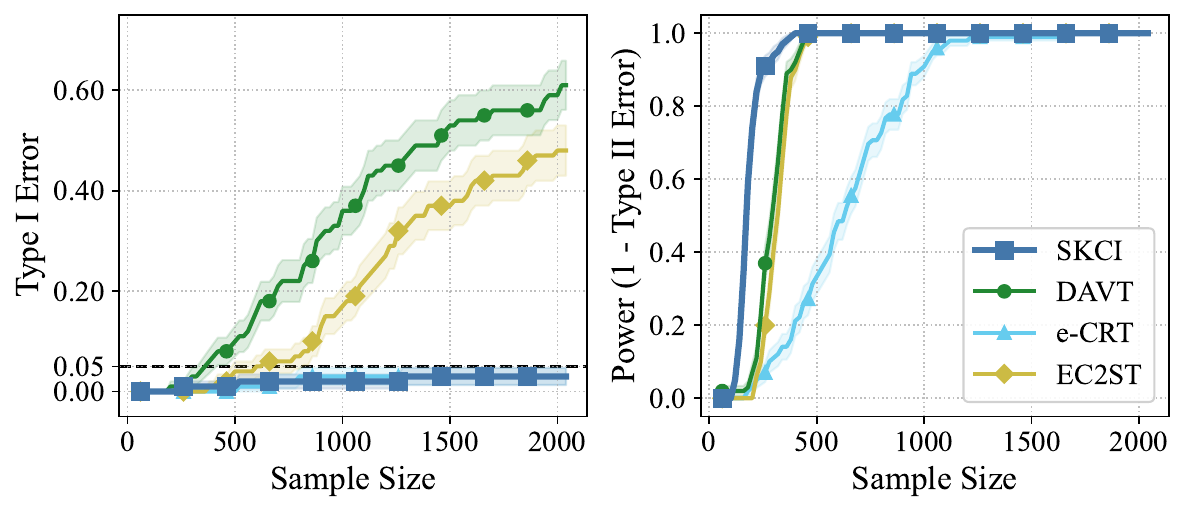}}
    \caption{
    1-dimensional data.
}
    \label{sind1-online}
  \end{center}
\end{subfigure}
\begin{subfigure}[ht]{\linewidth}
  \begin{center}
    \centerline{\includegraphics[width=\columnwidth]{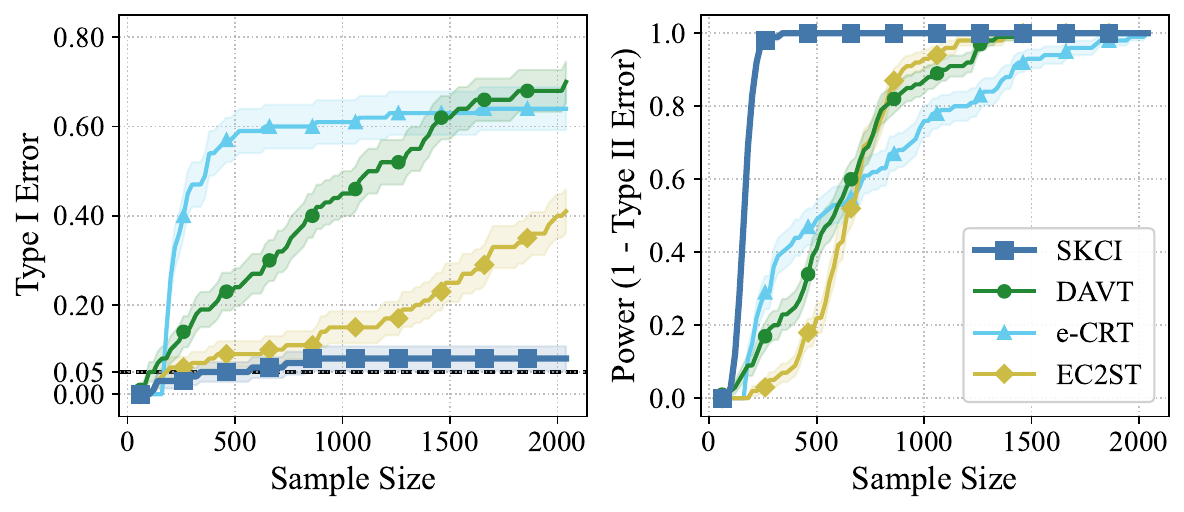}}
    \caption{
    3-dimensional data, shared coordinate.
}
    \label{sind000-online}
  \end{center}
\end{subfigure}
\begin{subfigure}[ht]{\linewidth}
  \begin{center}
    \centerline{\includegraphics[width=\columnwidth]{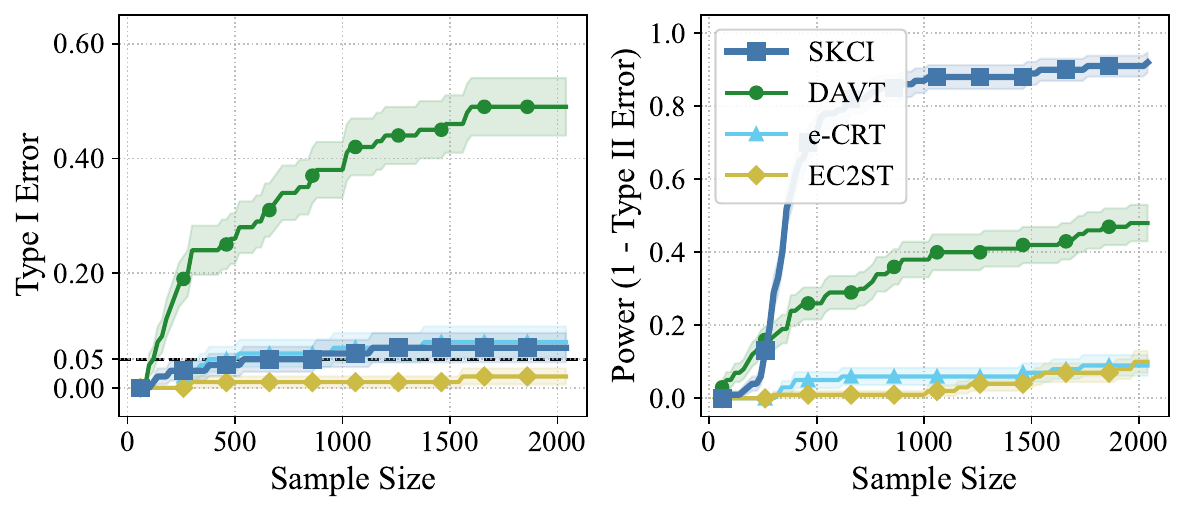}}
    \caption{
    3-dimensional data, separate coordinate.
}
    \label{sind012-online}
  \end{center}
\end{subfigure}
\caption{CI hardness data \citep{he2025hardness} in Online mode.}
\end{figure}

\paragraph{Linearly Dependent Gaussian Data}
We begin with a Gaussian benchmark commonly used in recent work on model-free conditional independence testing \citep{shaer2023model,pandeva2024davt}.
We sample $C \sim \N(0, \mathbf{I}_{19})$,
$u \sim \N(0, \mathbf{I}_{19})$, and
$A \mid (C, u) \sim \N(u^\top C, 1)$.
Under the null hypothesis, we define
$B \mid (A, C) \sim \N((w^\top C)^2, 1)$,
so that $A \ind B \mid C$ despite having strong nonlinear dependence on $C$.
Under the alternative, we introduce a linear dependence on $A$:
$B \mid (A, C) \sim \N((w^\top C)^2 + 3A, 1)$.

Figure \ref{fig:gaussian-online}, as well as \cref{gaussian-appendix} in the appendix,
show Type~I and power curves.
In Oracle mode, all methods control Type~I error and have reasonable power.
In Pretrained and Online modes, some baseline methods begin to suffer from worse Type I error rates,
while SKCI remains stable and rapidly achieves high power.

\paragraph{CI Hardness Benchmarks (1D \& 3D)}
The conditional independence hardness benchmarks of \citet{he2025hardness}
have $C$-varying dependence which is particularly difficult to detect.
Here $C \sim \mathcal N(0, I)$,
$A = \cos(e_A^\top C) + 0.1 r_A$,
and $B = \exp(e_B^\top C) + 0.1 r_B$,
where $(r_A,r_B)$ are jointly Gaussian with unit variance and conditional covariance
$\gamma(e_C^\top C)$.
Under $\Hzero$, $\gamma(c)=0$; under $\Hone$, $\gamma(c)=\sin(3c)$.

There are three configurations, of increasing difficulty:
(i) 1D, where $C \in \R$ and $e_A = e_B = e_C = 1$ (\cref{sind1-online,sind1-oracle,sind1-pretrained});
(ii) 3D with shared coordinates, $e_A = e_B = e_C$ (\cref{sind000-online,sind000-oracle,sind000-pretrained});
(iii) 3D with separate coordinates, where $e_A$, $e_B$, $e_C$ are orthonormal (\cref{sind012-online,sind012-oracle,sind012-pretrained}).

In the 1D and shared-coordinate settings, SKCI consistently matches or exceeds the power of competing methods while maintaining tight Type~I control.
In Oracle mode, all methods maintain controlled Type~I error,
but SKCI achieves the highest power;
other methods struggle with Type~I error and/or power in Pretrained and particularly Online mode.
The separate-coordinate 3D setting is substantially harder:
in all three modes,
other methods fail to detect dependence and/or suffer severe Type~I error inflation.
In contrast, SKCI's online kernel optimization successfully adapts to the relevant subspace, yielding reliable power and reasonable Type~I control even when the dependence signal is decoupled from marginal structure.

\begin{figure}[t]
    \centering
    \includegraphics[width=\columnwidth]{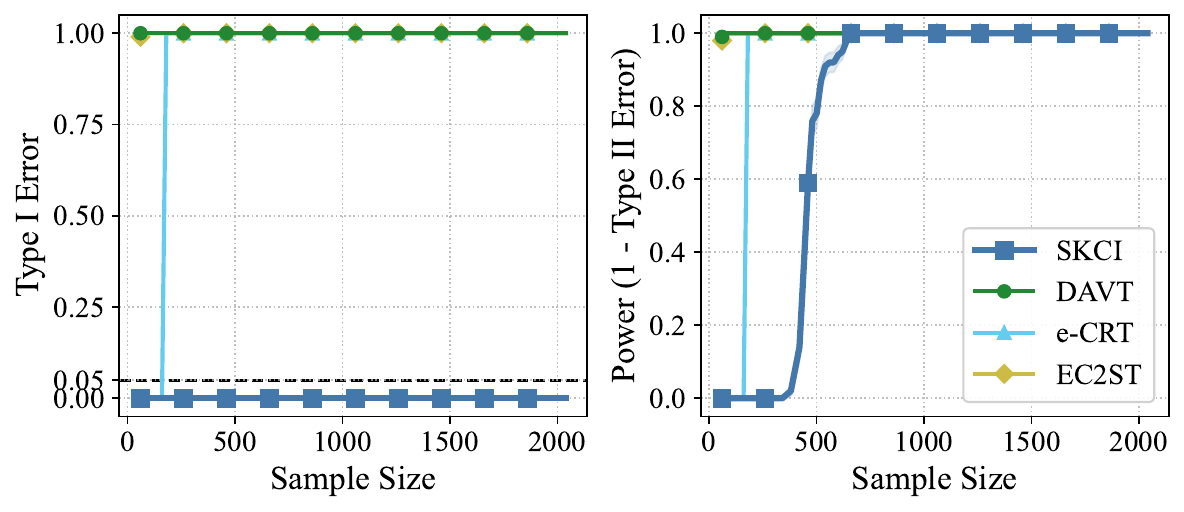}
    \caption{Synthetic neural data in Online mode.}
    \label{rats-online}
\end{figure}

\begin{figure}[h]
  \begin{center}
    {\includegraphics[width=\columnwidth]{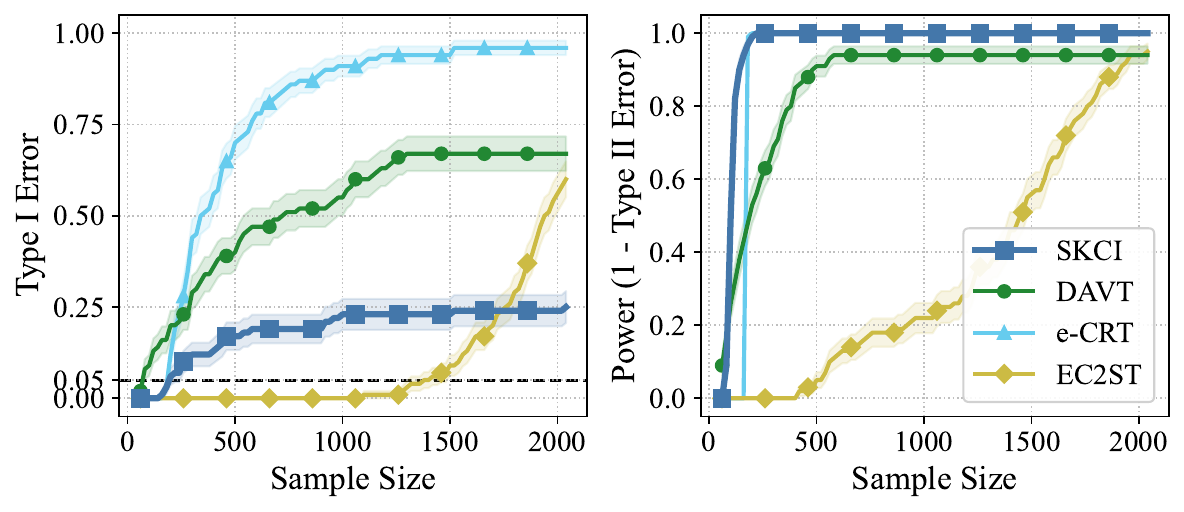}}
    \caption{
      dSprites data, Online mode.
    }
    \label{dsprites-online}
  \end{center}
\end{figure}

\begin{figure*}[h]
  \begin{center}
    \centerline{\includegraphics[width=\textwidth]{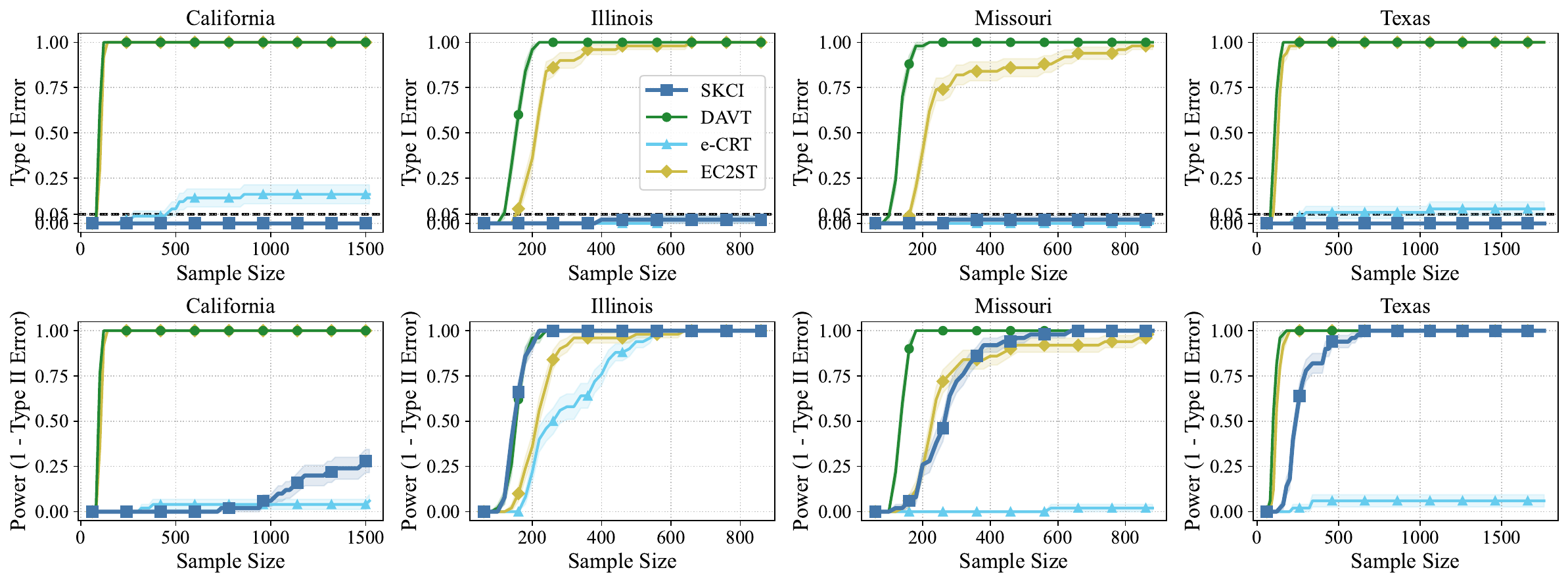}}
    \caption{
      Car insurance discrimination data, Online mode. Top rows give Type~I error, bottom rows give power for data across four states.
    }
    \vspace{-0.5em}
    \label{cars-online}
  \end{center}
\end{figure*}

\paragraph{Synthetic Neural Data (RatInABox)}
Following \citet{pogodin2024splitkci},
we evaluate SKCI in a high-dimensional, biologically motivated setting using the RatInABox simulator \citep{rat-in-a-box}.
The goal is to test whether head-direction cells ($A\in\R^{100}$) -- neurons that fire
as a function of the animal’s heading -- are conditionally independent of conjunctive
cells ($B\in\R^{100}$) -- which respond jointly to heading and spatial location --
given the animal’s physical state $C\in\R^4$ (position and orientation).
That is, we ask whether conjunctive activity is computed ``downstream'' of
head-direction signals (our $\Hone$) or is computed separately (our $\Hzero$),
despite common dependence on $C$.

\Cref{rats-pretrained,rats-online} reports results; exact distributions for Oracle mode are intractable.
In this challenging high-dimensional regime, SKCI achieves strong power while tightly controlling Type~I error.
EC2ST suffers extreme Type~I error,
as does DAVT in Online mode;
DAVT in Pretrained achieves almost no power.
e-CRT has good Type~I control but is much slower to detect dependence than SKCI.

\paragraph{Image Data (dSprites)}
Inspired by \citet{zhang2025testing},
we construct an image-based conditional independence task from dSprites \citep{dsprites17}.
Let \(A\) denote the shape of the object, \(B\) denote the full image, and
\(C\) denote a cropped view of \(B\).
When the cropped image \(C\) contains only part of the object, the full image
\(B\) carries additional information about the shape beyond \(C\),
so $A \nind B \mid C$.
If \(C\) contains the full object, \(B\) provides no additional shape information, and $A \ind B \mid C$.

\Cref{dsprites-online} shows that
SKCI maintains Type~I error substantially better than the competing baselines,
whose rejection rates quickly approach one even when the crop already contains
the full object.  In the dependent setting, all methods achieve high power with
moderate sample sizes, indicating that the main distinction on this task is null
calibration rather than sensitivity to the alternative.

\paragraph{Car Insurance Discrimination}
Finally, we apply SKCI to the car insurance dataset of \citet{angwin2017minority}, following the auditing protocol of \citet{polo2023conditional} and \citet{pogodin2024splitkci}.
We test whether insurance premiums ($A$) are conditionally independent of minority neighborhood status ($B$), given driver risk factors ($C$).
We only consider Online mode,
since data sizes are not large enough to support significant splits for Pretrained models.

To assess Type~I error, we construct a synthetic null by clustering $C$ and shuffling premiums within clusters, stratified by company and state. 
The results are computed from 50 runs per (state, company) pair, and the final decision is obtained via majority vote across companies within each state.
\cref{cars-online} reports rejection rates and power across four U.S.\ states in Online mode.
SKCI maintains conservative Type~I error across all states while achieving competitive or superior power.
DAVT and EC2ST have Type~I error rates near 1,
while e-CRT again is more controlled but underpowered.

\section{Discussion}
While {exact uniform Type~I control is impossible without assumptions},
our theory shows that SKCI's Type~I inflation can be controlled over finite
time horizons when the conditional mean embedding estimates are sufficiently
accurate. 
Although existing theoretical rates for  conditional mean embedding estimation can be slow \citep{li2022optimal,li:vvrls},
we find that, across several challenging synthetic and real-world benchmarks, 
SKCI tends to exhibit substantially better Type~I behavior
than prior methods while retaining competitive power for detecting conditional
dependence.
\section*{Acknowledgements}

The authors would like to thank Nathaniel Xu, Ilmun Kim, and the anonymous reviewers for several useful suggestions.

This work was enabled in part by support from the Natural Sciences and Engineering Research
Council of Canada, the Canada CIFAR AI Chairs Program, Calcul Québec, the BC DRI Group, and the
Digital Research Alliance of Canada.

\section*{Impact Statement}

Conditional independence testing is a tool with broad applications across machine learning and statistics, including some that are good for society and others that are more dubious.
We do not feel our work, which helps increase the reliability of an existing class of methods,
raises any particular consequences that must be discussed here.

\bibliography{refs}
\bibliographystyle{icml2026}

\newpage
\appendix
\crefalias{section}{appendix}
\crefalias{subsection}{appendix}
\crefalias{subsubsection}{appendix}
\onecolumn

\section{Proof of Type~I Error Bound}
\label{app:proof_type1}

We prove the finite-sample Type~I error bound for the sequential test under
\(\Hzero\).  Recall that the wealth process is updated as
\[
    W_t = W_{t-1}(1+\lambda_t \hV_t),
    \qquad
    \hV_t = \max\{\Vr_t-\hgamma_t,-1\},
\]
where \(\hgamma_t\) is the implemented Gaussian calibration shift.  If
\(\hgamma_t\) were the exact null-calibrating shift for the conditional law of
\(\Vr_t\mid\F_{t-1}\), then the conditional expected payoff would be zero and
the resulting wealth process would be an exact martingale.  In finite samples,
however, \(\hgamma_t\) is only an approximate shift, so the wealth process may
have a nonzero conditional drift.

We define this drift by
\[
    \delta_t
    \defeq
    \E_{\Hzero}\!\left[
        \max\{\Vr_t-\hgamma_t,-1\}
        \,\middle|\,
        \F_{t-1}
    \right].
\]
Throughout this section, all expectations and distributions are conditional on
\(\F_{t-1}\) unless stated otherwise.

The proof proceeds in four parts.  First, we decompose \(\delta_t\) into a
finite-block Gaussian approximation term and a Gaussian calibration mismatch
term.  Second, we establish the sensitivity of the Gaussian null-calibrating
shift with respect to its mean and variance parameters.  Third, we bound the
mean and variance mismatch induced by conditional mean embedding estimation and
by the implemented variance scale.  Finally, we combine these bounds to obtain
a deterministic one-step drift bound and use it to construct a compensated
supermartingale, from which the finite-sample Type~I error bound follows by
Ville's inequality.

\subsection{Decomposition of the Supermartingale Drift}
\label{apx:decomposition}

In this section, we decompose the conditional drift of the betting payoff under
the null.  The payoff depends on the null-calibrating shift \(\hgamma_t\), which
is selected using the Gaussian approximation \(N(\bmu_t,\bsigma_t^2)\).  The
implemented statistic is based on a finite block of null-calibrating samples and
on the data-adaptive kernel \(h^{(t)}\), whose construction involves estimated
conditional mean embeddings.  We therefore separate the drift into a Gaussian
calibration term and a residual approximation term.  This decomposition allows
us to identify the conditions under which the betting process remains close to
a null supermartingale.

Let \(\bE_{\Hzero}[\cdot]\) and \(\bVar_{\Hzero}[\cdot]\) denote conditional expectation and variance under the null
distribution given the filtration \(\F_{t-1}\). 
Define the conditional drift
of the truncated payoff by
\[
    \delta_t
    \defeq
    \E_{\Hzero}\bigl[\max\{\Vr_t-\hgamma_t,-1\} \mid \F_{t-1} \bigr]
    =
    \bE_{\Hzero}\bigl[\max\{\Vr_t-\hgamma_t,-1\}\bigr].
\]
We decompose this drift by adding and subtracting the corresponding expectation
under the Gaussian approximation.  
Let \(X_t\) denote a conditionally Gaussian random variable such that
\[
    X_t \mid \F_{t-1}
    \sim
    N(\bmu_t,\bsigma_t^2).
\]
Then
\begin{align}
\delta_t
&=
\underbrace{
    \bE_{\Hzero}\bigl[\max\{\Vr_t-\hgamma_t,-1\}\bigr]
    -
    \bE_{X_t\sim N(\bmu_t,\bsigma_t^2)}\bigl[\max\{X_t-\hgamma_t,-1\}\bigr]
}_{\textnormal{(I) Gaussian approximation error}}
+
\underbrace{
    \bE_{X_t\sim N(\bmu_t,\bsigma_t^2)}\bigl[\max\{X_t-\hgamma_t,-1\}\bigr]
}_{\textnormal{(II) Gaussian calibration error}} .
\label{eq:drift-decomposition}
\end{align}
The first term measures the discrepancy between the true conditional null law
of \(\Vr_t\) and  its limiting Gaussian law \(N(\bmu_t,\bsigma_t^2)\).   
This term captures finite-sample effects.  
The second term is the drift obtained when the implemented shift \(\hgamma_t\)
is evaluated under this Gaussian limiting law.  Since \(\hgamma_t\) is selected
using the centered calibration law \(\N(0,\hsigma_t^2)\), this term reflects the
calibration mismatch between \(N(0,\hsigma_t^2)\) and
\(N(\bmu_t,\bsigma_t^2)\).  The ideal finite-sample shift \(\gamma_t^*\) would
instead be calibrated against the true conditional null law of
\(\Vr_t\mid\F_{t-1}\).
The decomposition
above isolates the error due to this approximation from the error due to the
choice of the Gaussian-calibrated shift \(\hgamma_t\).

We now describe the Gaussian approximation used in
\eqref{eq:drift-decomposition}.  Conditional on \(\F_{t-1}\), the historical
training samples \(\{x_i\}_{i=1}^n\subseteq \Xtr_{t-1}\), the statistic
\(S_n(\Xtr_{t-1})\), and the kernel \(h^{(t)}\) are fixed.  For a fresh null
sample \(y_j\), define
\[
    g^{(t)}(y_j)
    \defeq
    \frac{1}{n\bigl(S_n(\Xtr_{t-1})+\varepsilon\bigr)}
    \sum_{i=1}^n h^{(t)}(x_i,y_j).
\]
Then the raw betting statistic can be written as the empirical average
\[
    \Vr_t
    =
    \frac{1}{b}\sum_{j=1}^b g^{(t)}(y_j).
\]
Under the null and conditional on \(\F_{t-1}\), the variables
\(\{g^{(t)}(y_j)\}_{j=1}^b\) are i.i.d.  We define their conditional mean and
variance by
\[
    m_t
    \defeq
    \bE_{\Hzero}\bigl[g^{(t)}(Y)\bigr],
    \qquad
    s_t^2
    \defeq
    \bVar_{\Hzero}\bigl(g^{(t)}(Y)\bigr),
\]
where \(Y\) denotes a fresh null draw.  Consequently,
\[
    \bE_{\Hzero}[\Vr_t] = m_t,
    \qquad
    \bVar_{\Hzero}(\Vr_t) = \frac{s_t^2}{b}.
\]
Thus the natural Gaussian approximation is
\[
    \Vr_t \mid \F_{t-1}
    \approx
    \N\left(m_t,\frac{s_t^2}{b}\right).
\]
Equivalently, in the notation of \eqref{eq:drift-decomposition},
\[
    \bmu_t = m_t,
    \qquad
    \bsigma_t^2 = \frac{s_t^2}{b}.
\]
By the conditional central limit theorem, provided
\(0<s_t^2<\infty\),
\[
    \frac{\sqrt b(\Vr_t-m_t)}{s_t}
    \;\xrightarrow{d}\;
    \N(0,1)
    \qquad
    \text{conditionally on } \F_{t-1}.
\]
Therefore, for large \(b\), the conditional null law of \(\Vr_t\mid\F_{t-1}\)
is approximated by
\(
    \N(\bmu_t,\frac{s_t^2}{b}),
\)
the same conditional Gaussian law used to define \(X_t\).

\paragraph{Bounding Term I (Asymptotic Gap)}
We bound the Asymptotic Gap using the Wasserstein metric. Define the test function 
$\ell(v) \defeq \max\{v-\hgamma_t, -1\}$.

\begin{lemma}[Lipschitz Continuity of the Payoff]
The function $\ell(v)$ is Lipschitz continuous with constant $L=1$.
\end{lemma}
\begin{proof}
For any $x, y \in \R$,
\[
|\ell(x) - \ell(y)| = |\max\{x-\hgamma_t, -1\} - \max\{y-\hgamma_t, -1\}|
\le
|(x-\hgamma_t) - (y-\hgamma_t)|
= |x-y|.
\qedhere\]
\end{proof}

By Kantorovich-Rubinstein duality, the Wasserstein-1 distance between distributions $P$ and $Q$ is given by
\[
W_1(P, Q) = \sup_{f \in \text{Lip}(1)} \abs[\Big]{ \E_{X\sim P}[f(X)] - \E_{Y\sim Q}[f(Y)] }
.\]
Since $\ell \in \text{Lip}(1)$,
we immediately have that
\begin{equation}
\left| \text{Term (I)} \right|
=
\left| \bE_\Hzero[\ell(\Vr_t)] - \bE_{\N(\bmu_t, \bsigma^2_t)}[\ell(X_t)] \right|
\le
W_1\left(P_{\Hzero}{(\Vr_t \in \cdot\mid \F_{t-1})}, \, \N(\bmu_t, \bsigma^2_t)\right).
\end{equation}

Conditioning on \( \F_{t-1}\), let 
\[
    Z_t
    \defeq
    \frac1n\sum_{i=1}^n h^{(t)}(x_i,Y) .
\]
By Berry-Esseen-type Wasserstein bounds, such as Corollary 4.2 of \citet{stein},
\[
W_1\left(
    P_{\Hzero}(\Vr_t\in\cdot\mid \F_{t-1}), \,
    \N(\bmu_t,\bsigma_t^2)
\right)
\le
\frac{C_1\rho_t}{b\bigl(S_n(\Xtr_{t-1})+\varepsilon\bigr)}
\le
\frac{C_1\rho_t}{b\varepsilon},
\]
where the Berry–Esseen moment ratio is defined as
\[
    \rho_t
    \defeq
    \frac{
        \bE_{\Hzero}\!\left[
            |Z_t-\bE_{\Hzero}[Z_t]|^3
        \right]
    }{
        \bVar_{\Hzero}(Z_t)
    },
\]
and \(C_1>0\) is a universal constant.
Since we assume that the conditional third central moment of \(Z\) is uniformly bounded and that its conditional variance is uniformly
nondegenerate, the Berry--Esseen moment ratio $\rho_t$ is uniformly bounded.  We therefore assume that there exists
\(\rho<\infty\) such that \(\rho_t\le \rho\) for all \(t\le T\).
Thus, 
\begin{equation}
\left| \text{Term (I)} \right|
\le
\frac{C_1\rho}{b\, \varepsilon}.
\end{equation}

\paragraph{Term II: Calibration mismatch.}
The second term is the drift obtained when the implemented shift \(\hgamma_t\)
is evaluated under the Gaussian approximation to the conditional null law of
\(\Vr_t\).  This term is zero or negative if \(\hgamma_t\) is at least as large
as the ideal calibrating shift for this Gaussian law.  Therefore, bounding this
term reduces to controlling the difference between the implemented shift
\(\hgamma_t\), selected using \(N(0,\hsigma_t^2)\), and the Gaussian-calibrating
shift associated with \(N(\bmu_t,\bsigma_t^2)\).

We defer the quantitative bound to Appendix~\ref{apx:shift}, where we introduce
the Gaussian payoff function
\[
    f(\gamma;\mu,\sigma)
    \defeq
    \E_{X\sim N(\mu,\sigma^2)}
    \bigl[\max\{X-\gamma,-1\}\bigr].
\]
Denote \(\gamma(\bmu_t,\bsigma_t)\) the Gaussian calibrating shift for
\(N(\bmu_t,\bsigma_t^2)\).
With this notation,
\begin{align*}
    \textnormal{Term (II)}
    &=
    \bE_{X_t\sim N(\bmu_t,\bsigma_t^2)}\bigl[\max\{X_t-\hgamma_t,-1\}\bigr] \\
    &=
    f(\hgamma_t;\bmu_t,\bsigma_t) 
    - \underbrace{f(\gamma(\bmu_t,\bsigma_t);\bmu_t,\bsigma_t)}_{=0}.
\end{align*}
This term is controlled by the discrepancy between
\(\hgamma_t\) and \(\gamma(\bmu_t,\bsigma_t)\).  In \cref{apx:shift}, we bound
the sensitivity of the calibrating map \((\mu,\sigma)\mapsto\gamma(\mu,\sigma)\).
Combining this sensitivity bound with the mean and variance mismatch bounds
from \cref{apx:mean-var} gives the desired bound on Term~(II).

\subsection{Sensitivity of the Null-Calibrating Shift}
\label{apx:shift}
We will use $\N(\mu, \sigma^2)$ to denote a Gaussian distribution with mean $\mu$ and variance $\sigma^2$,
$\Phi(\cdot; \mu, \sigma^2)$ for its CDF,
and $\phi(\cdot; \mu, \sigma^2)$ for its PDF.
$\Phi(\cdot)$ and $\phi(\cdot)$ refer to the case where $\mu = 0$, $\sigma^2 = 1$.

Define the functional
$f : \R \times \R \times \R_{>0} \to \R$
representing the expected truncated payoff under a Gaussian law by
\begin{equation}
f(\gamma;\mu,\sigma)
\defeq
\E_{X \sim \N(\mu,\sigma^2)}
\bigl[\max\{X-\gamma,-1\}\bigr].
\label{eq:implict-func}
\end{equation}
The function $f$ is strictly decreasing in $\gamma$. Therefore, the ideal null-calibrating shift $\gamma(\mu,\sigma)$ for a random variable $X \sim \N(\mu,\sigma^2)$ is defined as the unique solution to $f(\gamma;\mu,\sigma)=0$.

\begin{lemma}[Implicit Sensitivity of the Shift]
\label{lem:ift_appendix}
The mapping $(\mu,\sigma)\mapsto\gamma(\mu,\sigma)$ is continuously differentiable
on $\R\times\R_{>0}$ and satisfies the following properties:
\begin{itemize}
\item \textbf{Mean sensitivity:}
\(
\frac{\partial \gamma}{\partial \mu} = 1.
\)

\item \textbf{Variance sensitivity:}
The variance derivative
\(
\frac{\partial \gamma}{\partial \sigma}
\)
is both positive
and monotonically increasing in $\sigma$.

\item \textbf{Lipschitz bound:}
For any $(\mu_1,\sigma_1)$ and $(\mu_2,\sigma_2)$,
\[
\gamma(\mu_1,\sigma_1)-\gamma(\mu_2,\sigma_2)
\le
(\mu_1-\mu_2)
+
\Lambda(\sigma_1)(\sigma_1-\sigma_2),
\]
where
\(
\Lambda(\sigma)\defeq \frac{\phi(\xi_\sigma)}{\Phi(-\xi_\sigma)}
\)
and $\xi_\sigma$ is the unique solution to
\(
\sigma\,[\phi(\xi)-\xi\Phi(-\xi)]=1.
\)
\end{itemize}
\end{lemma}

\begin{proof}
Since the expectation involves truncation at $-1$, the function $f$ admits the 
representation
\begin{align*}
f(\gamma;\mu,\sigma)
&=
\int_{-\infty}^{\gamma-1} (-1)\,\phi(v;\mu,\sigma)\,dv
+
\int_{\gamma-1}^{\infty} (v-\gamma)\,\phi(v;\mu,\sigma)\,dv
\\
&=
\sigma\,
\phi\!\left(\frac{\gamma-\mu-1}{\sigma}\right)
-
(\gamma-\mu-1)\,
\Phi\!\left(-\frac{\gamma-\mu-1}{\sigma}\right)
-1.
\end{align*}

The function $f$ is continuously differentiable and strictly decreasing in
$\gamma$.
Its partial derivatives are
\begin{equation*}
\frac{\partial f}{\partial \gamma}
=
-\Phi\!\left(-\frac{\gamma-\mu-1}{\sigma}\right),
\quad
\frac{\partial f}{\partial \mu}
=
\Phi\!\left(-\frac{\gamma-\mu-1}{\sigma}\right),
\quad
\frac{\partial f}{\partial \sigma}
=
\phi\!\left(\frac{\gamma-\mu-1}{\sigma}\right).
\end{equation*}
Let $\xi\defeq(\gamma(\mu,\sigma)-\mu-1)/\sigma$.
Since $\partial f/\partial\gamma \in (-1, 0)$ for all $\mu$ and all $\sigma>0$, the implicit function
theorem guarantees that $\gamma(\mu,\sigma)$ is continuously differentiable, with
\begin{equation}
\frac{\partial \gamma}{\partial \mu}
=
-\frac{\partial f/\partial\mu}{\partial f/\partial\gamma}
=1,
\qquad
\frac{\partial \gamma}{\partial \sigma}
=
-\frac{\partial f/\partial\sigma}{\partial f/\partial\gamma}
=
\frac{\phi(\xi)}{\Phi(-\xi)}.
\end{equation}

From the defining equation $f(\gamma;\mu,\sigma)=0$, we obtain
\[
1
=
\sigma\,[\phi(\xi)-\xi\Phi(-\xi)]
\quad\Longleftrightarrow\quad
\frac{1}{\sigma}=r(\xi),
\]
where $r(\xi)\defeq\phi(\xi)-\xi\Phi(-\xi)$.
Since $r'(\xi)=-\Phi(-\xi)<0$, the function $r$ is strictly decreasing in $\xi$.
Hence the solution $\xi_\sigma$ to $1/\sigma=r(\xi)$ is unique and is strictly increasing in $\sigma$.

Moreover, the ratio
${\phi(\xi)}/{\Phi(-\xi)}$
is strictly increasing in $\xi$. Since
${\partial \gamma}/{\partial \sigma}
=
{\phi(\xi)}/{\Phi(-\xi)}$,
it follows that $\partial \gamma/\partial \sigma$ is strictly increasing in $\xi$.

Let $\xi_\sigma$ denote the unique solution to $1/\sigma=r(\xi)$, so that $\xi$ is implicitly determined by $\sigma$. Define
\[
\Lambda(\sigma)\defeq \frac{\partial \gamma}{\partial \sigma}
=
\frac{\phi(\xi_\sigma)}{\Phi(-\xi_\sigma)}.
\]

Then, since $\xi_\sigma$ is strictly increasing in $\sigma$ and $\phi(\xi)/\Phi(-\xi)$ is strictly increasing in $\xi$, it follows that $\Lambda(\sigma)$ is strictly increasing in $\sigma$.

By the mean value theorem, there exists
$(\tilde{\mu},\tilde{\sigma})$ on the line segment connecting
$(\mu_1,\sigma_1)$ and $(\mu_2,\sigma_2)$ such that
\[
\gamma(\mu_1,\sigma_1)-\gamma(\mu_2,\sigma_2)
=
(\mu_1-\mu_2)
+
\Lambda(\tilde{\sigma})(\sigma_1-\sigma_2).
\]
If $\sigma_1\ge\sigma_2$, then $\tilde{\sigma}\le\sigma_1$ and hence
$\Lambda(\tilde{\sigma})\le\Lambda(\sigma_1)$.
If $\sigma_1<\sigma_2$, then $\sigma_1-\sigma_2<0$ and
$\Lambda(\tilde{\sigma})>\Lambda(\sigma_1)$, which again yields
\(
\Lambda(\tilde{\sigma})(\sigma_1-\sigma_2)
\le
\Lambda(\sigma_1)(\sigma_1-\sigma_2).
\)
In both cases, the claimed bound follows.
\end{proof}

We have shown that $\frac{\partial f}{\partial \gamma} \in (-1,0)$
in the proof of \cref{lem:ift_appendix}.
Let $\gamma(\bmu_t, \bsigma_t)$ be the unique solution to $f(\gamma;\bmu_t,\bsigma_t)=0$.
When
$\gamma(\bmu_t, \bsigma_t)\ge \hgamma_t$, 
then
\[
f(\hgamma_t;\bmu_t,\bsigma_t) - f(\gamma(\bmu_t, \bsigma_t);\bmu_t,\bsigma_t)
\le 
\gamma(\bmu_t, \bsigma_t) - \hgamma_t 
.\]
When
$\gamma(\bmu_t, \bsigma_t) < \hgamma_t$, 
then
\[
f(\hgamma_t;\bmu_t,\bsigma_t) - f(\gamma(\bmu_t, \bsigma_t);\bmu_t,\bsigma_t)
\le 
0.\]
Thus, 
\[
\text{Term II} \le \max\{\gamma(\bmu_t, \bsigma_t)- \hgamma_t, 0 \}
.\]

Recall that \(\hgamma_t\) is selected using the Gaussian law \(N(0,\hsigma_t^2)\).
Combining this with the Lipschitz bound from \cref{lem:ift_appendix} gives
\[
\text{Term II}
\le
\max\{\bmu_t + \Lambda(\bsigma_t)(\bsigma_t-\hsigma_t), 0\}.
\]
This bound controls the calibration error by the mean and variance mismatch
between the Gaussian approximation to the conditional null law of \(\Vr_t\) and the Gaussian law used to select the implemented shift \(\hgamma_t\).

\subsection{Bounding the Mean and Variance Mismatch}
\label{apx:mean-var}

In this section, we bound the mismatch between the Gaussian approximation
\(N(\bmu_t,\bsigma_t^2)\) to the conditional null law of \(\Vr_t\) and the
centered Gaussian law \(N(0,\hsigma_t^2)\) used to select the implemented shift
\(\hgamma_t\).  Here \(\bmu_t\) and \(\bsigma_t^2\) denote the scalar mean and
variance parameters of the Gaussian approximation to \(\Vr_t\mid\F_{t-1}\).
They are distinct from the conditional mean embeddings \(\muac\) and \(\mubc\),
which are Hilbert-space valued functions used to construct the residualized
kernel \(h^{(t)}\).

We control this mismatch through two quantities:
\[
    \abs{\bmu_t}
    \qquad\text{and}\qquad
    \abs{\bsigma_t^2-\hsigma_t^2}.
\]
The first term captures the nonzero mean induced by conditional mean embedding
estimation error, while the second captures both conditional mean embedding
error and finite-batch variation in the variance estimate.

Define the regression errors
\[
    \detact(c) \defeq \muact(c)-\muac(c),
    \qquad
    \detbct(c) \defeq \mubct(c)-\mubc(c).
\]

\begin{lemma}[Bound on the Mean Mismatch]
\label{lem:mean_mismatch}
Assume the kernel is bounded such that $\sup_{x} {h^{(t)}(x, x)} \le \kappa$,
and \(\sup_c\norm{\phi_C(c)}\le 1\).  
Then
\[
    \abs{\bmu_t}
    \le
    \frac{\sqrt{\kappa}}{\varepsilon}
    \norm{\detact}\norm{\detbct}.
\]
\end{lemma}

\begin{proof}
For simplicity, we write \(S_n(\Xtr_{t-1})\) as  \(S_n^{(t-1)}\).
Recall that
\begin{align}
\bmu_t
=
\bE_{\Hzero}[\Vr_t]
=
\frac{1}{S_n^{(t-1)}+\varepsilon}
\bE_{\Hzero}\left[
    \frac1n\sum_{i=1}^n h^{(t)}(x_i,Y)
\right].
\label{eq:scalar-mean-recall}
\end{align}
Here \(x_i=(a_i,b_i,c_i)\) are fixed training samples conditional on \(\F_{t-1}\), while
\(Y=(A,B,C)\) denotes a fresh null draw.

Using the feature representation
\[
    h^{(t)}(x_i,Y)
    =
    \bigl\langle
        \psi^{(t)}(x_i),\psi^{(t)}(Y)
    \bigr\rangle_{\H},
\]
where
\[
    \psi^{(t)}(x)
    \defeq
    \bigl(\phia(a)-\muact(c)\bigr)
    \otimes
    \bigl(\phib(b)-\mubct(c)\bigr)
    \otimes
    \phic(c),
\]
is in the space of Hilbert-Schmidt (HS) operators $\H=\HS(\H_\C, \HS(\H_\B, \H_\A))$.

Expanding
\[
    \phia(A)-\muact(C)
    =
    \bigl(\phia(A)-\muac(C)\bigr)-\detact(C),
\]
and similarly for \(B\), we note that \(\muact\), \(\mubct\), and hence
\(\detact\), \(\detbct\), are \(\F_{t-1}\)-measurable.

Under \(\Hzero\), \(A\ind B\mid C\).  Moreover,
\[
    \bE_{\Hzero}\bigl[\phia(A)-\muac(C)\mid C\bigr]=0,
    \qquad
    \bE_{\Hzero}\bigl[\phib(B)-\mubc(C)\mid C\bigr]=0.
\]
Therefore, after taking the conditional expectation given \(C\), all terms
containing at least one population residual vanish.  The only remaining term is
the product of the regression errors.  Hence
\[
    \bE_{\Hzero}\bigl[\psi^{(t)}(Y)\mid \F_{t-1}\bigr]
    =
    \bE_C
    \Big[
        \detact(C)\otimes \detbct(C)\otimes \phic(C)
    \Big].
\]

Consequently,
\begin{align}
\bmu_t
&=
\frac{1}{S_n^{(t-1)}+\varepsilon}
\left\langle
    \frac1n\sum_{i=1}^n \psi^{(t)}(x_i),
    \bE_C
    \Big[
        \detact(C)\otimes \detbct(C)\otimes \phic(C)
    \Big]
\right\rangle_{\H}.
\label{eq:mean-error-inner-product}
\end{align}

Applying Cauchy--Schwarz and the triangle inequality,
\begin{align*}
|\bmu_t|
&\le
\frac{1}{S_n^{(t-1)}+\varepsilon}
\left\|
    \frac1n\sum_{i=1}^n \psi^{(t)} (x_i)
\right\|_{\H}
\left\|
    \bE_C
    \big[
        \detact(C)\otimes \detbct(C)\otimes \phic(C)
    \big]
\right\|_{\H}
\\
&\le
\frac{1}{S_n^{(t-1)}+\varepsilon}
\left(
    \frac1n\sum_{i=1}^n \norm{\psi^{(t)} (x_i)}_{\H}
\right)
\bE_C
\big[
    \norm{\detact(C)}
    \norm{\detbct(C)}
    \norm{\phic(C)}
\big].
\end{align*}
Since
\[
    \norm{\psi^{(t)} (x_i)}_{\H}
    =
    \sqrt{h^{(t)}(x_i,x_i)}
    \le \sqrt{\kappa},
\]
and \(\sup_c\norm{\phic(c)}\le 1\), we have
\[
\bE_C
\big[
    \norm{\detact(C)}
    \norm{\detbct(C)}
    \norm{\phic(C)}
\big]
\le
\norm{\detact}\norm{\detbct}.
\]
Here we used
\[
    \norm{\detact(C)}
    \le
    \norm{\detact}\norm{\phic(C)}
    \le
    \norm{\detact},
\]
and similarly for \(\detbct\). Therefore,
\[
    |\bmu_t|
    \le
    \frac{\sqrt{\kappa}}{S_n^{(t-1)}+\varepsilon}
    \norm{\detact}\norm{\detbct}.
\]
Since \(S_n^{(t-1)}\) is the V-statistic estimator of a squared RKHS norm,
we have \(S_n^{(t-1)}\ge 0\).  Hence
\(
    S_n^{(t-1)}+\varepsilon \ge \varepsilon .
\)
It follows that
\[
    |\bmu_t|
    \le
    \frac{\sqrt{\kappa}}{\varepsilon}
    \norm{\detact}\norm{\detbct}.
\]
\end{proof}

\begin{lemma}[Bound on the Variance Mismatch]
\label{lem:var_mismatch}
Assume the kernel is bounded such that $\sup_{x} {h^{(t)}(x, x)} \le \kappa$.  
The difference between the $\bsigma_t^2$ and $\hsigma_t^2$ is bounded by
\begin{equation}
    \abs{\bsigma_t^2 - \hsigma_t^2}
    \le
    \frac{2\kappa^2}{b\varepsilon^2}.
\end{equation}
\end{lemma}

\begin{proof}

  Recall that the raw betting statistic can be written as
\[
    \Vr_t
    =
    \frac1b\sum_{j=1}^b g^{(t)}(Y_j).
\]
Under the null and conditional on \(\F_{t-1}\), the variables
\(\{g^{(t)}(Y_j)\}_{j=1}^b\) are i.i.d.  Hence the conditional variance of
\(\Vr_t\) is
\[
    \bsigma_t^2
    =
    \frac1b
    \bVar_{\Hzero}\bigl(g^{(t)}(Y)\bigr)
    =
    \frac1b
    \bE_{\Hzero}\bigl[(g^{(t)}(Y))^2\bigr]
    -
    \frac1b \bmu_t^2,
\]
where \(Y\) is a fresh null draw and
\[
    \bmu_t
    =
    \bE_{\Hzero}\bigl[g^{(t)}(Y)\bigr].
\]
The empirical variance-scale estimator used for calibration is
\[
    \hsigma_t^2
    =
    \frac1b
    \left(
        \frac1b\sum_{j=1}^b (g^{(t)}(v_j))^2
    \right)
    =
    \frac1{b^2}\sum_{j=1}^b (g^{(t)}(v_j))^2,
\]
where \(\{v_j\}_{j=1}^b\) are validation samples independent of the training
data used to construct \(g^{(t)}\).

Therefore,
\begin{align*}
\abs{\bsigma_t^2-\hsigma_t^2}
&\le \abs{\bsigma_t^2}+\abs{\hsigma_t^2}\\
&\le \frac1b \left\{\bE_{\Hzero}\bigl[(g^{(t)}(Y))^2\bigr] + \frac1b\sum_{j=1}^b (g^{(t)}(v_j))^2\right\}.
\end{align*}
Moreover, by the boundedness assumption,
\(
\norm{\psi^{(t)} (x_i)}_{\H}
    =
    \sqrt{h^{(t)}(x_i,x_i)}
    \le \sqrt{\kappa}
\).
Hence
\begin{align*}
\abs{g^{(t)}(Y)}
&=
\left|
\frac{1}{n\bigl(S_n(\Xtr_{t-1})+\varepsilon\bigr)}
\sum_{i=1}^n h^{(t)}(x_i,Y)
\right| \\
&\le
\frac{1}{n\varepsilon}
\sum_{i=1}^n
\abs{h^{(t)}(x_i,Y)} \\
&=
\frac{1}{n\varepsilon}
\sum_{i=1}^n
\Big\lvert
\left\langle
    \psi^{(t)}(x_i),\psi^{(t)}(Y)
\right\rangle_{\H}
\Big\rvert \\
&\le
\frac{1}{n\varepsilon}
\sum_{i=1}^n
\norm{\psi^{(t)}(x_i)}_{\H}
\norm{\psi^{(t)}(Y)}_{\H} \\
&\le
\frac{\kappa}{\varepsilon}.
\end{align*}
Similarly, \(\abs{g^{(t)}(v_j)}\le \frac{\kappa}{\varepsilon}\) for each \(j\in \{0,\dots, b\}\).
Therefore,
\[
    \abs{\bsigma_t^2-\hsigma_t^2}
    \le
    \frac{2\kappa^2}{b\varepsilon^2}.
\]
This completes the proof.
\end{proof}

We now return to Term~(II). By the shift sensitivity bound,
\[
\textnormal{Term (II)}
\le
\max\{
    \bmu_t+\Lambda(\bsigma_t)(\bsigma_t-\hsigma_t),
    0
\}.
\]
Hence
\begin{align*}
\textnormal{Term (II)}
&\le \abs{\bmu_t+\Lambda(\bsigma_t)(\bsigma_t-\hsigma_t)}\\
&\le
\abs{\bmu_t}
+
\Lambda(\bsigma_t)\abs{\bsigma_t-\hsigma_t} \\
&=
\abs{\bmu_t}
+
\Lambda(\bsigma_t)
\frac{\abs{\bsigma_t^2-\hsigma_t^2}}{\bsigma_t+\hsigma_t} \\
&\le
\abs{\bmu_t}
+
\frac{\Lambda(\bsigma_t)}{\bsigma_t}
\abs{\bsigma_t^2-\hsigma_t^2}.
\end{align*}

Recall that
\(
    \Lambda(\bsigma_t)
    =
    \frac{\phi(\xi_{\bsigma_t})}{\Phi(-\xi_{\bsigma_t})},
\)
where \(\xi_{\bsigma_t}\) solves
\(
    \phi(\xi)-\xi\Phi(-\xi)=\frac{1}{\bsigma_t}.
\)
Thus
\[
    \frac{\Lambda(\bsigma_t)}{\bsigma_t}
    =
    \frac{\phi(\xi_{\bsigma_t})^2}{\Phi(-\xi_{\bsigma_t})}
    -
    \xi_{\bsigma_t}\phi(\xi_{\bsigma_t}),
\]
which is uniformly bounded above by a finite constant \(C_2\). Combining this
with the mean and variance mismatch bounds gives
\begin{align*}
\textnormal{Term (II)}
&\le
\abs{\bmu_t}
+
C_2\abs{\bsigma_t^2-\hsigma_t^2} \\
&\le
\frac{\sqrt{\kappa}}{\varepsilon}
\norm{\detact}\norm{\detbct}
+
\frac{2C_2\kappa^2}{b\varepsilon^2}.
\end{align*}

\subsection{Finite-Sample Type~I Error Bound}
\label{apx:type1_error}

In this section, we translate the one-step drift bound into a finite-sample
bound on the Type~I error under \(\Hzero\).  Combining the Gaussian
approximation bound for Term~(I) with the calibration mismatch bound for
Term~(II), we obtain
\begin{equation}
\label{eq:total_drift_bound}
\delta_t
=
\textnormal{Term (I)}+\textnormal{Term (II)}
\le
\frac{C_1\rho}{b \,\varepsilon}
+
\frac{\sqrt{\kappa}}{\varepsilon}
\norm{\detact}\norm{\detbct}
+
\frac{2C_2 \kappa^2}{b\, \varepsilon^2}.
\end{equation}
Let \(U_t\) denote the right-hand side of
\cref{eq:total_drift_bound}.

\begin{theorem}[Finite-Sample Type~I Error]
\label{thm:finite_sample_type1}
Let \(\alpha\in(0,1)\). Suppose that under \(\Hzero\), for every
\(t\le T\),
\[
    \delta_t
    =
    \E_{\Hzero}[V_t\mid \F_{t-1}]
    \le
    U_t,
\]
where
\[
    U_t
    \defeq
    \frac{C_1\rho}{b\,\varepsilon}
    +
    \frac{\sqrt{\kappa}}{\varepsilon}
    \norm{\detact}\norm{\detbct}
    +
    \frac{2C_2\kappa^2}{b\,\varepsilon^2}.
\]
Then
\[
\Pr_{\Hzero}\!\left(
    \exists t\le T:
    W_t\ge \frac1\alpha
\right)
\le
\alpha
\exp\left(
    \sum_{t=1}^T \lambda_t U_t
\right).
\]
\end{theorem}
\begin{proof}
Define the compensated process
\[
    \widetilde W_t
    \defeq
    \frac{W_t}{
        \prod_{i=1}^t(1+\lambda_i U_i)
    }.
\]
Since \(V_t\ge -1\) and \(\lambda_t\in(0,1)\), the wealth process is
nonnegative. Moreover, by the assumed drift bound,
\begin{align*}
\E_{\Hzero}[\widetilde W_t\mid \F_{t-1}]
&=
\frac{
    \E_{\Hzero}[W_t\mid \F_{t-1}]
}{
    \prod_{i=1}^t(1+\lambda_i U_i)
} \\
&=
\frac{
    W_{t-1}
    \E_{\Hzero}[1+\lambda_t V_t\mid \F_{t-1}]
}{
    \prod_{i=1}^t(1+\lambda_i U_i)
} \\
&=
\frac{
    W_{t-1}(1+\lambda_t\delta_t)
}{
    \prod_{i=1}^t(1+\lambda_i U_i)
} \\
&\le
\frac{
    W_{t-1}(1+\lambda_t U_t)
}{
    \prod_{i=1}^t(1+\lambda_i U_i)
}
=
\widetilde W_{t-1}.
\end{align*}
Thus \((\widetilde W_t)_{t=0}^T\) is a nonnegative supermartingale with
\(\widetilde W_0=1\).

If \(W_t\ge 1/\alpha\), then
\[
    \widetilde W_t
    =
    \frac{W_t}{\prod_{i=1}^t(1+\lambda_i U_i)}
    \ge
    \frac{1}{
        \alpha\prod_{i=1}^t(1+\lambda_i U_i)
    }.
\]
Equivalently,
\[
    W_t
    =
    \widetilde W_t
    \prod_{i=1}^t(1+\lambda_i U_i).
\]
Since
\[
    \prod_{i=1}^t(1+\lambda_i U_i)
    \le
    \exp\left(\sum_{i=1}^t \lambda_i U_i\right)
    \le
    \exp\left(\sum_{i=1}^T \lambda_i U_i\right),
\]
we have
\[
    W_t
    \le
    \widetilde W_t
    \exp\left(\sum_{i=1}^T \lambda_i U_i\right).
\]
Therefore,
\[
    W_t\ge \frac1\alpha
    \quad\Longrightarrow\quad
    \widetilde W_t
    \ge
    \frac{
        \exp\left(-\sum_{i=1}^T\lambda_i U_i\right)
    }{\alpha}.
\]
By Ville's inequality,
\begin{align*}
\Pr_{\Hzero}\!\left(
    \exists t\le T:
    W_t\ge \frac1\alpha
\right)
&\le
\Pr_{\Hzero}\!\left(
    \exists t\le T:
    \widetilde W_t
    \ge
    \frac{
        \exp\left(-\sum_{i=1}^T\lambda_i U_i\right)
    }{\alpha}
\right) \\
&\le
\alpha
\exp\left(\sum_{i=1}^T\lambda_i U_i\right).
\end{align*}
This proves the claim.
\end{proof}

\section{Implementation Details}
\label{apx:implementation}
\subsection{Datasets}\label{apx:implementation:datasets}

\paragraph{Synthetic neural data.}
Following \citet{pogodin2024splitkci}, we evaluate the tests on
high-dimensional synthetic neural data generated with RatInABox
\citep{rat-in-a-box}.  The setup follows the implementation of
\citet{pogodin2024splitkci}, available at
\href{https://github.com/romanpogodin/kernel-ci-testing}
{\nolinkurl{github.com/romanpogodin/kernel-ci-testing}}, and is based on the
head-direction and conjunctive-cell models from the
\href{https://github.com/RatInABox-Lab/RatInABox/blob/main/demos/conjunctive_gridcells_example.ipynb}
{RatInABox conjunctive-cells demo}.

We take \(A\) to be the activity of 100 head-direction cells and \(B\) to be the
activity of 100 conjunctive cells.  The conditioning variable \(C\) consists of
the two-dimensional head-direction vector and the two-dimensional position of
the agent.  The resulting testing question is whether the head-direction-cell
activity \(A\) contains information about the conjunctive-cell activity \(B\)
beyond what is explained by head direction and position \(C\).  As in
\citet{pogodin2024splitkci}, the activity is subsampled based on the noise
autocorrelation to obtain approximately i.i.d. observations.

\paragraph{dSprites.}
Inspired by \citet{zhang2025testing},
we construct a conditional independence task from dSprites \citep{dsprites17}.  Each observation is
a \(64 \times 64\) image containing a single object.  We fix the object scale and
use the shape label as \(A\).  The variable \(B\) is the full image.  The conditioning variable
\(C\) is a \(28 \times 28\) crop derived from \(B\), padded or placed on a fixed
canvas.

The two testing regimes differ in how the crop is chosen.  In the Type~I
setting, \(C\) is obtained from the tight bounding box of the object and hence
contains the full object.  In this case, once \(C\) is given, \(B\) should not
provide substantial additional information about the shape, so
\(A \perp\!\!\!\perp B \mid C\) should hold.  In the Type~II setting, the crop is
randomly shifted away from the object bounding box, so \(C\) typically contains
only partial object information.  Then \(B\) contains additional information
about \(A\) beyond \(C\), and the conditional independence null is false.

\paragraph{Car insurance data.}
Following the experiments of \citet{polo2023conditional,pogodin2024splitkci}, we also evaluate the
methods on a car insurance dataset originally collected by
\citet{angwin2017minority} from four US states and multiple insurance
companies.\footnote{%
Data available from \url{https://projects.propublica.org/graphics/carinsurance}.
}
The dataset contains three variables: car insurance price \(A\), minority
neighborhood indicator \(B\), and driver risk \(C\).  The minority
neighborhood indicator is defined as more than \(66\%\) non-white in California
and Texas, and more than \(50\%\) non-white in Missouri and Illinois.  The risk
variable \(C\) is constructed from driver-related risk factors, with additional
equalization of driver characteristics; see \citet{angwin2017minority} for
details.  The resulting conditional independence question is whether insurance
price remains associated with the minority-neighborhood indicator after
conditioning on driver risk.

\subsection{Kernel Choices in SKCI} \label{apx:implementation:kernels}

SKCI employs separate kernels for variables $A$, $B$, and $C$.
Unless otherwise stated, all kernels are fixed throughout a run and updated only
through improved conditional mean estimates as more data become available.

For the kernels on \(A\) and \(B\), we use an RBF kernel when the variables are
low-dimensional.  The bandwidth is fixed to \(1\) for synthetic datasets and
initialized using a variance heuristic for the car insurance data.  The
conditioning kernels \(k_{C \to A}\) and \(k_{C \to B}\) are selected by
leave-one-out kernel ridge regression with per-dimension lengthscales, with
early stopping based on a validation set.  We use the same kernel family for
conditional mean estimation and for the kernel on \(C\).

For Gaussian data, conditional means are estimated using an MLP feature map with
layer sizes \([128,64,32,8]\), followed by an RBF kernel.

For dSprites, we use convolutional feature maps before applying RBF kernels.
The full image \(B\) is encoded by a pretrained autoencoder for \(64\times 64\)
single-channel images, with hidden channel widths \([32,64,128]\) and a
\(16\)-dimensional latent representation.  The autoencoder is trained on the
same training samples used for the corresponding regression step and is then
frozen during kernel fitting.  Kernels involving \(B\) are computed on this
latent representation, with the RBF bandwidth initialized by the variance
heuristic.
The conditioning variable \(C\) is a \(28\times 28\) single-channel crop.  For
the kernel on \(C\), and for the conditional mean maps \(C\to A\) and \(C\to B\),
we use trainable CNN feature maps with hidden channel widths \([32,64,128]\)
and a fully connected hidden layer of width \(128\).  The output dimensions are
\(8\) for the kernel on \(C\), \(1\) for the \(C\to A\) map, and \(8\) for the
\(C\to B\) map.

For the remaining datasets, we apply RBF, Kronecker, or linear kernels directly,
following the dataset-specific kernel choices used in
\citet{pogodin2024splitkci,he2025hardness}.

\begin{table}[ht]
\centering
\caption{Kernel choices across datasets.}
\label{tab:kernels}
\begin{tabular}{lccc}
\toprule
Dataset & $k_A$ / $k_B$ & $k_C$ / $k_{C\to A}$ / $k_{C\to B}$ & Notes \\
\midrule
Gaussian & RBF & RBF + MLP & 1-dimensional A, B \\
Hardness (1D/3D) & RBF & RBF & 1-dimensional A, B \\
RatInABox & Linear & RBF & 100-dimensional $A,B$ \\
dSprites  & RBF / RBF + AutoEncoder & RBF + CNN & 1-dimensional A; image B, C \\
Car Insurance & RBF / Kronecker & RBF & 1-dimensional A; categorical $B$ \\
\bottomrule
\end{tabular}
\end{table}

\subsection{Conditional Mean Estimation Modes}

We evaluate SKCI under three estimation regimes:
\begin{itemize}
    \item \textbf{Oracle Mode}: The true conditional mean $\E[A \mid C]$ is provided for $C\to A$ regression task as the target.
    \item \textbf{Pretrained Mode}: Conditional mean estimators are trained using 3{,}000 auxiliary samples independent of the testing stream.
    \item \textbf{Online Mode}: Conditional mean estimators are updated sequentially using the accumulated training history.
\end{itemize}

\subsection{Baselines and Fair Comparison} \label{apx:baselines}

We compare SKCI against e-CRT \citep{shaer2023model}, DAVT
\citep{pandeva2024davt} and EC2ST \cite{pandeva2024evaluating}.

For baseline methods, all regressors assume a Gaussian conditional model
$A \mid C \sim \N(\mu(C), \sigma^2)$.
In Oracle mode, both $\mu(C)$ and $\sigma^2$ are provided; otherwise, they are
estimated from data.
Hyperparameters are selected using an additional set of 1{,}000 samples,
with the search space shown in \cref{tab:reg_search}, via 5-fold cross-validation
over 5 independent data seeds.
Final architecture choices and tuned hyperparameters for each dataset
(Gaussian, Hardness 1D/3D, RatInABox, dSprites, and Car Insurance)
are reported in \cref{tab:configs-cond}.

\begin{table}[ht]
\centering
\caption{Hyperparameter search space for baseline models.}
\label{tab:reg_search}
\begin{tabular}{ll}
\toprule
Parameter & Values \\
\midrule
Learning rate & $\{0.1, 0.01, 0.001, 0.0001\}$ \\
Weight decay & $\{0, 10^{-5}, 10^{-4}\}$ \\
Hidden channels & $\{[32], [32, 64], [32, 64, 128], [32, 64, 128, 256]\}$ \\
Hidden layers & $\{[32], [64], [128], [64,32], [128,64], [128,64,32]\}$ \\
\bottomrule
\end{tabular}
\end{table}

\begin{table}[h]
\centering
\caption{Model configurations for conditional mean estimation.}
\label{tab:configs-cond}
\begin{tabular}{lccccc}
\toprule
Dataset & Model & Hidden channels & Hidden dims & LR & Weight decay \\
\midrule
Gaussian & MLP & -- & [32] & 0.01 & 0 \\
CI Hardness (1D) & MLP & -- & [128] & 0.01 & 0 \\
CI Hardness (3D, shared) & MLP & -- & [64,32] & 0.01 & $10^{-4}$ \\
CI Hardness (3D, sep.) & MLP & -- & [32] & 0.01 & 0 \\
RatInABox & MLP & -- & [128,64] & 0.01 & 0 \\
dSprites & CNN & [32, 64] & [32] & 0.001 & 0 \\
Car insurance & MLP & -- & [128,64] & 0.01 & 0 \\
\bottomrule
\end{tabular}
\end{table}

In addition to conditional mean estimation, baseline models also rely on
discriminator or regression models whose hyperparameters are tuned separately.
For DAVT, the discriminator distinguishing $Z$ from $\tilde Z$ 
selected based on validation loss aligned with the wealth objective
$g(Z)-g(\tilde Z)$.
For e-CRT, the discriminator uses regression loss with $(A,C)$ as input and $B$ as target.
The corresponding search spaces are the same as the search space shown in \cref{tab:reg_search},
while DAVT additionally searches learning rates in
$\{0.01, 0.005, 0.001, 0.0005, 0.0001\}$.
Dataset-specific configures for discriminator architectures and optimization
parameters are reported in \cref{tab:configs-disc}.

For the image-based dSprites experiment, we use CNN-based models.  For DAVT and
EC2ST, the network processing \(B\) uses hidden channel widths
\([32,64,128]\), while the network processing the conditioning information uses
hidden channel widths \([32,64,128,256]\).  For e-CRT, the model uses an
encoder--decoder architecture, with encoder hidden channel widths \([32,64]\)
and decoder hidden channel widths \([64,32]\).

\begin{table}[h]
\centering
\caption{Discriminator configurations for each dataset.}
\label{tab:configs-disc}
\begin{tabular}{lcccc}
\toprule
Dataset & Method & Hidden dims & LR & Weight decay \\
\midrule
Gaussian & DAVT & [128] & 0.0005 & 0 \\
CI Hardness (1D) & DAVT & [128] & 0.001 & $10^{-5}$ \\
CI Hardness (3D, shared) & DAVT & [128] & 0.005 & 0 \\
CI Hardness (3D, separate) & DAVT & [128] & 0.005 & 0 \\
RatInABox & DAVT & [128] & 0.0005 & $10^{-5}$ \\
dSprites & DAVT & [64] & 0.001 & $10^{-4}$ \\
Car insurance & DAVT & [128] & 0.0005 & $10^{-5}$ \\
\midrule
Gaussian & e-CRT & [128,64] & 0.01 & $10^{-5}$ \\
CI Hardness (1D) & e-CRT & [128,64,32] & 0.01 & $10^{-4}$ \\
CI Hardness (3D, shared) & e-CRT & [128,64,32] & 0.01 & $10^{-4}$ \\
CI Hardness (3D, separate) & e-CRT & [128,64] & 0.01 & $10^{-5}$ \\
RatInABox & e-CRT & [128] & 0.01 & $10^{-4}$ \\
dSprites & e-CRT & [32] & 0.01 & $10^{-4}$ \\
Car insurance & e-CRT & [128] & 0.01 & $10^{-4}$ \\
\midrule
Gaussian & EC2ST & [128] & 0.0005 & 0 \\
CI Hardness (1D) & EC2ST & [128] & 0.001 & $10^{-5}$ \\
CI Hardness (3D, shared) & EC2ST & [128] & 0.005 & 0 \\
CI Hardness (3D, separate) & EC2ST & [128] & 0.005 & 0 \\
RatInABox & EC2ST & [128] & 0.0005 & $10^{-5}$ \\
dSprites & EC2ST & [64] & 0.001 & $10^{-4}$ \\
Car insurance & EC2ST & [128] & 0.0005 & $10^{-5}$ \\
\bottomrule
\end{tabular}
\end{table}

\section{Additional Results}
\label{apx:exp-results}

\subsection{Results for Oracle and Pretrained Mode}

\Cref{gaussian-appendix,sind1-comparison,sind000-comparison,sind012-comparison}
give emiprical results for the synthetic problems in Oracle and Pretrained mode.

\begin{figure}[ht]
  \centering
  \begin{subfigure}{0.49\columnwidth}
    \centering
    \includegraphics[width=\textwidth]{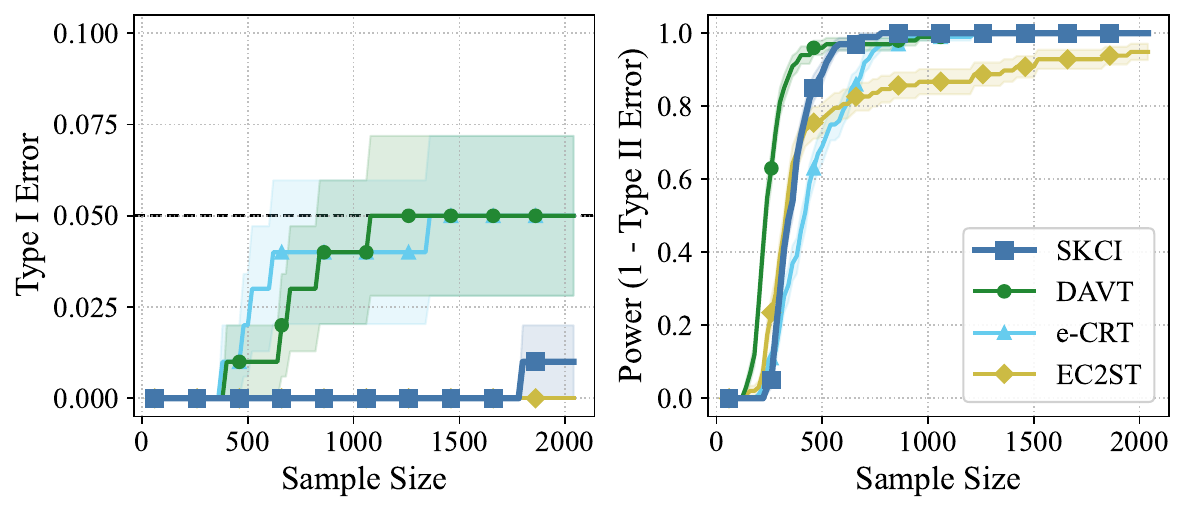}
    \caption{Oracle Mode}
  \end{subfigure}
  \hfill
  \begin{subfigure}{0.49\columnwidth}
   \centering
   \includegraphics[width=\textwidth]{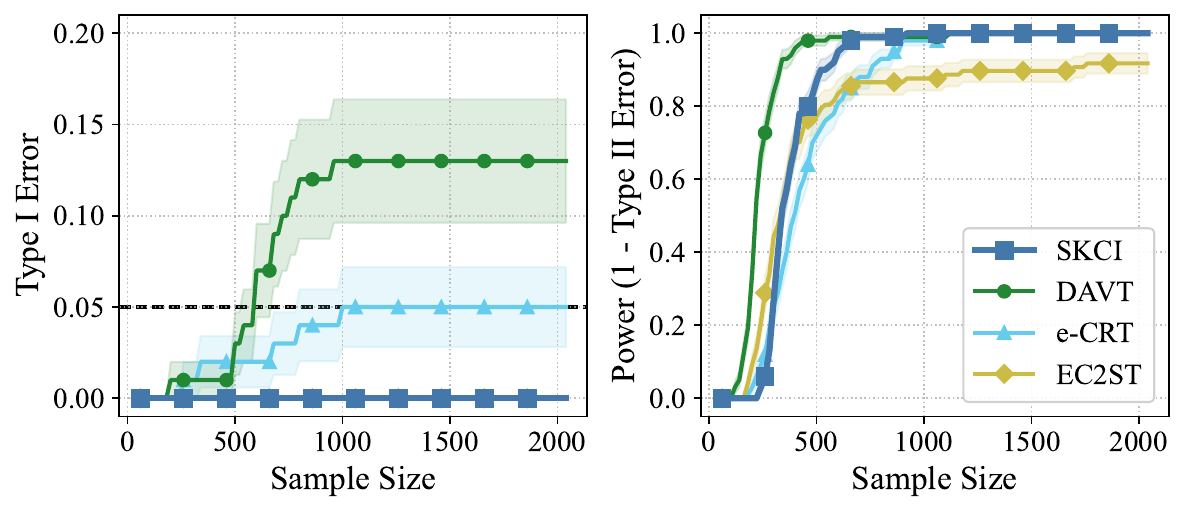}
   \caption{Pretrained Mode}
   \label{gaussian-pretrained}
  \end{subfigure}
  \caption{
    Gaussian data. 
  }
  \label{gaussian-appendix}
\end{figure}

\begin{figure}[ht]
  \centering
  \begin{subfigure}{0.49\columnwidth}
    \centering
    \includegraphics[width=\textwidth]{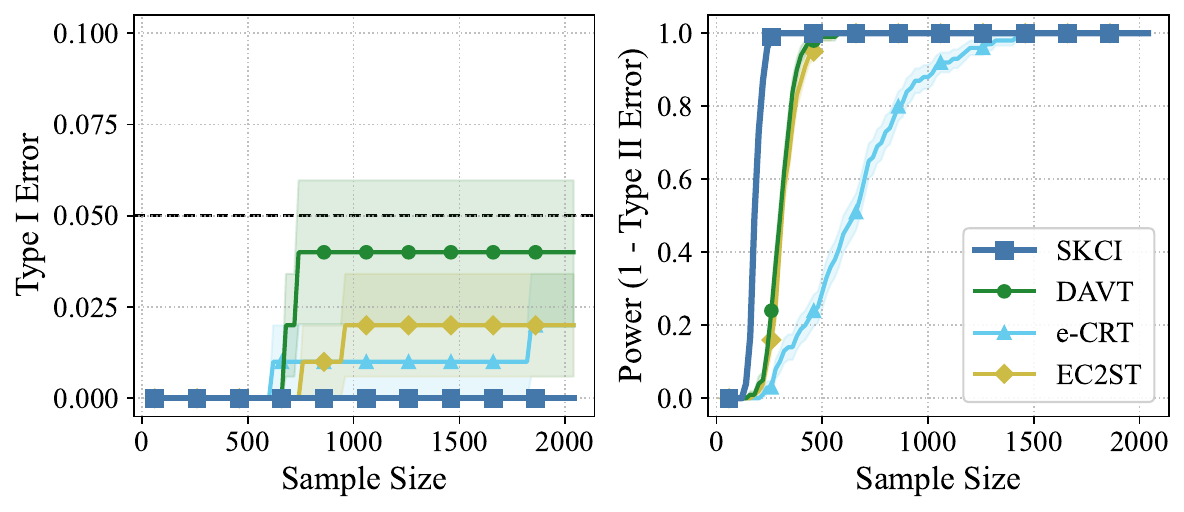}
    \caption{Oracle Mode}
    \label{sind1-oracle}
  \end{subfigure}
  \hfill
  \begin{subfigure}{0.49\columnwidth}
    \centering
    \includegraphics[width=\textwidth]{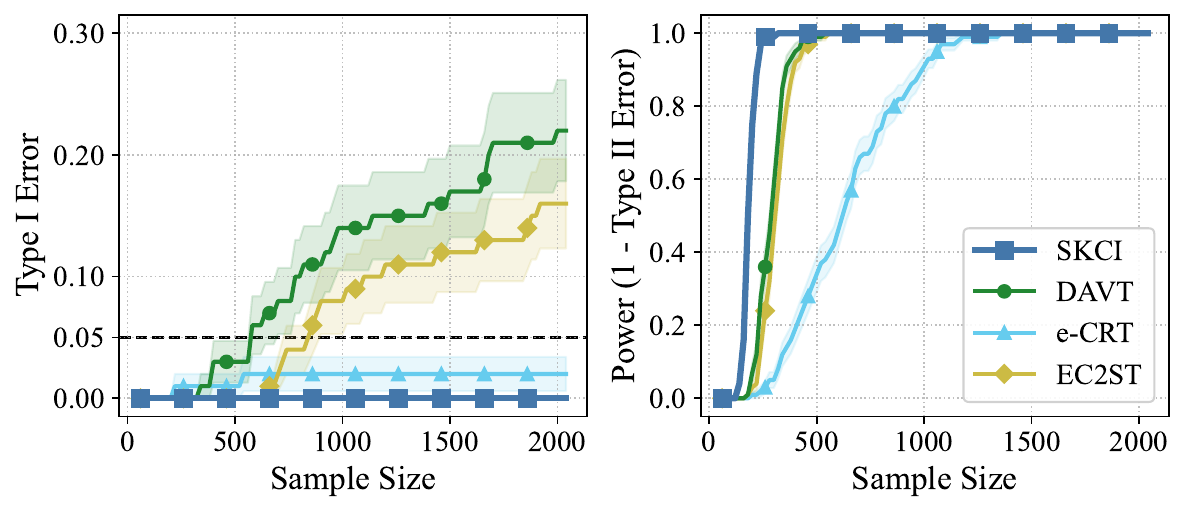}
    \caption{Pretrained Mode}
    \label{sind1-pretrained}
  \end{subfigure}
  \caption{
    CI Hardness data (1D). 
  }
  \label{sind1-comparison}
\end{figure}

\begin{figure}[h]
  \centering
  \begin{subfigure}{0.49\columnwidth}
    \centering
    \includegraphics[width=\textwidth]{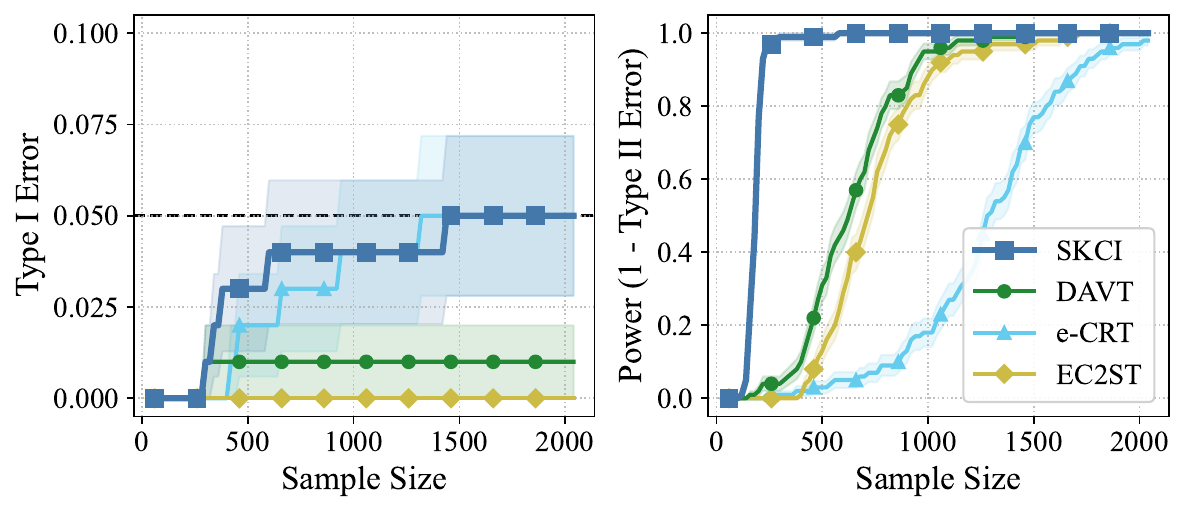}
    \caption{Oracle Mode}
    \label{sind000-oracle}
  \end{subfigure}
  \hfill
  \begin{subfigure}{0.49\columnwidth}
    \centering
    \includegraphics[width=\textwidth]{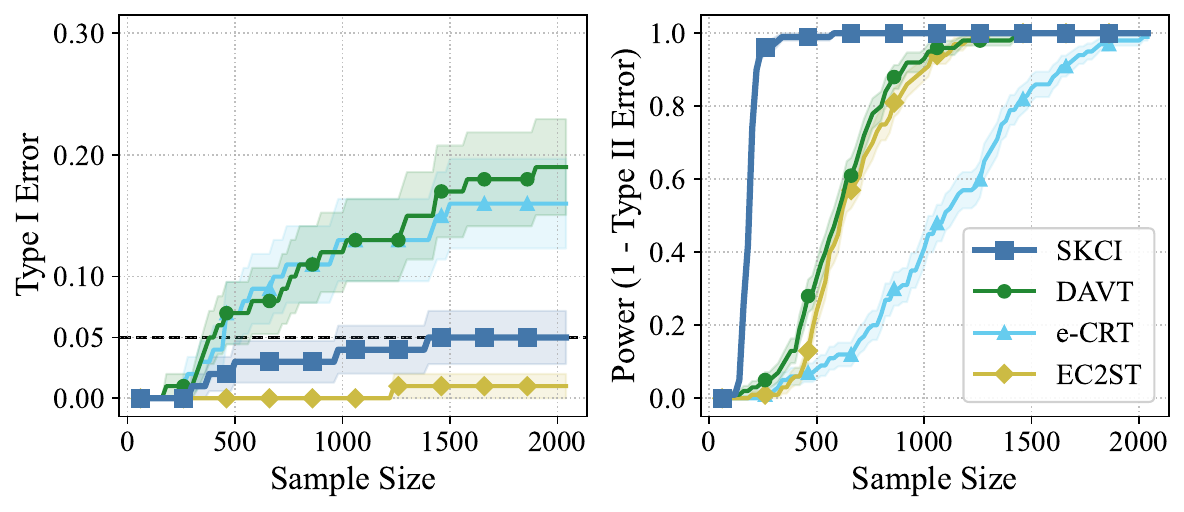}
    \caption{Pretrained Mode}
    \label{sind000-pretrained}
  \end{subfigure}
  \caption{
    CI Hardness data (3D, shared coordinate). 
  }
  \label{sind000-comparison}
\end{figure}

\begin{figure}[h]
  \centering
  \begin{subfigure}{0.49\columnwidth}
    \centering
    \includegraphics[width=\textwidth]{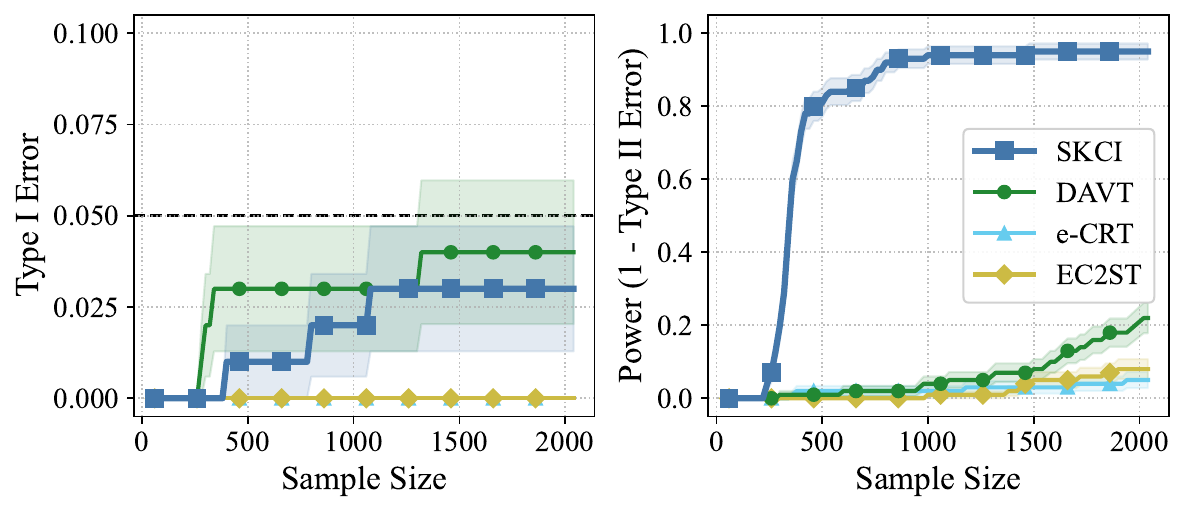}
    \caption{Oracle Mode}
    \label{sind012-oracle}
  \end{subfigure}
  \hfill
  \begin{subfigure}{0.49\columnwidth}
    \centering
    \includegraphics[width=\textwidth]{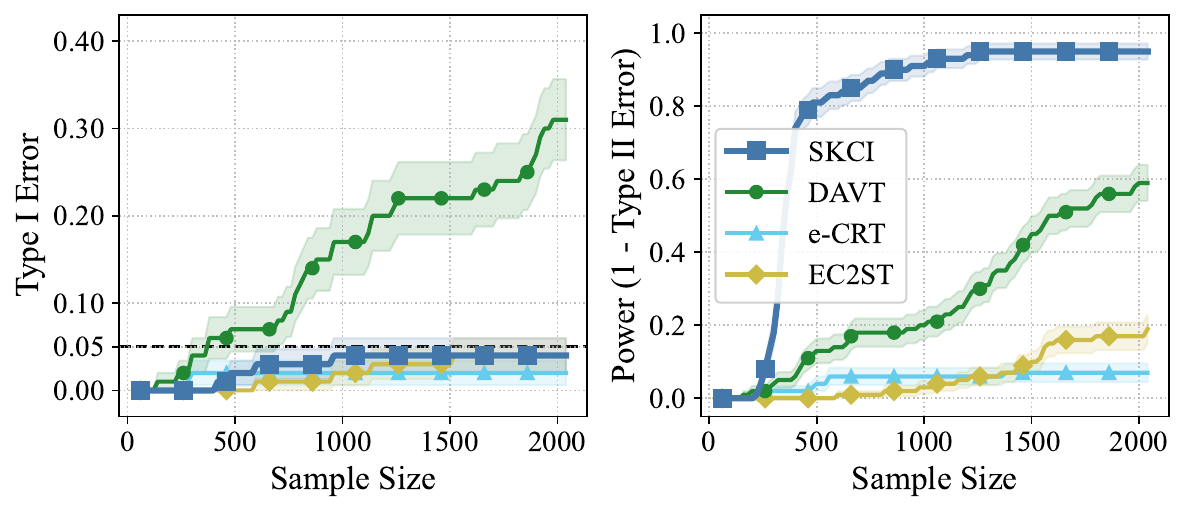}
    \caption{Pretrained Mode}
    \label{sind012-pretrained}
  \end{subfigure}
  \caption{
    CI Hardness data (3D, separate coordinate). 
  }
  \label{sind012-comparison}
\end{figure}

\begin{figure}[h]
  \begin{center}
    \centerline{\includegraphics[width=0.5\columnwidth]{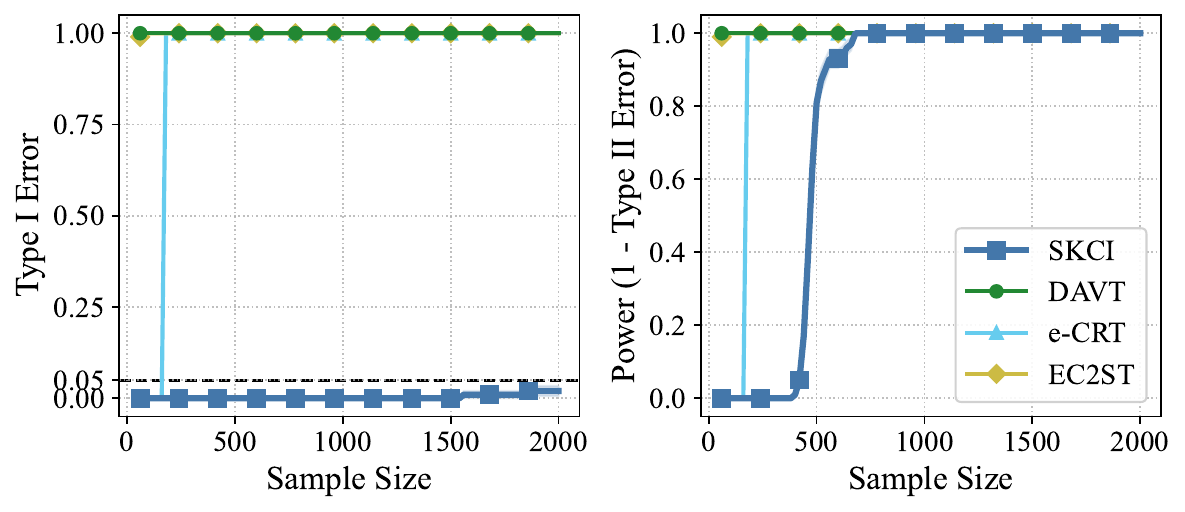}}
    \caption{
Synthetic neural data, Pretrained Mode.
}
    \label{rats-pretrained}
  \end{center}
\end{figure}

\subsection{Long Time Horizon Experiments}
\label{apx:longT}
Figure~\ref{ab-sin-longT} reports long-horizon runs under the null, in which we keep the
hyperparameters fixed and only increase the time horizon \(T\).  The cumulative
number of rejections stabilizes as the sample size increases.  In both settings,
most rejections occur early in the online sequence, and the reject count remains
nearly constant after a moderate number of samples.  This suggests that, when
the conditional mean estimates are sufficiently accurate, the rejection count
does not necessarily continue to grow with the time horizon.
\begin{figure}[ht]
  \centering
  \begin{subfigure}{0.3\columnwidth}
    \centering
    \includegraphics[width=\textwidth]{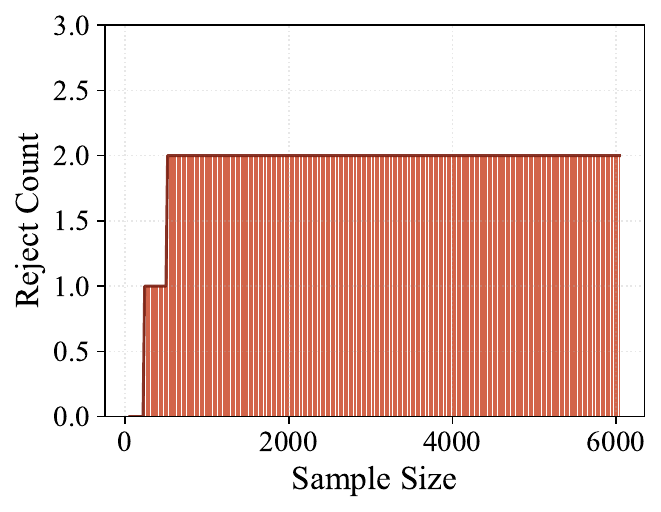}
    \caption{1-dimensional data.}
    \label{ab-sind1-T}
  \end{subfigure}
  \begin{subfigure}{0.3\columnwidth}
    \centering
    \includegraphics[width=\textwidth]{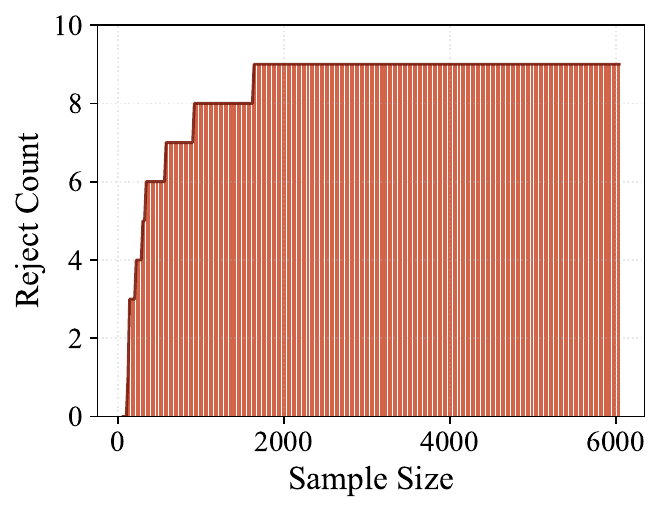}
    \caption{3-dimensional, shared coordinate.}
    \label{ab-sind000-T}
  \end{subfigure}
  \caption{
    Long horizon experiments on CI hardness data under the null, Online mode, 100 runs.
  }
  \label{ab-sin-longT}
\end{figure}

\subsection{Ablation}
\label{apx:ablation}
\paragraph{Batch Size}
Figure \ref{ab-sind1-bs} shows that smaller batch sizes accelerate power accumulation but may slightly increase variance in early rounds.
The default choice $b=20$ offers a good balance between responsiveness and stability.

\paragraph{Regularization Parameter}
Figure \ref{ab-sind1-eps} illustrates the effect of the denominator regularization $\varepsilon$.
Too small values may lead to unstable payoffs, while overly large values reduce power.
Across experiments, we fix $\varepsilon=10^{-6}$ which provides robust performance.

\paragraph{Shift Parameter}
Figure~\ref{ab-sind3-gamma} examines the effect of using a more
conservative shift parameter \(\gamma\).  Since the payoff is nonincreasing in
\(\gamma\), increasing \(\gamma\) reduces the size of the wealth update and can
therefore make the test more conservative.  We compare the default choice
\(\hat\gamma_t\) with the enlarged value \(\hat\gamma_t+0.1\).  The enlarged
choice reduces Type~I error while leaving power essentially unchanged in this
experiment, suggesting that the procedure is not highly sensitive to small
conservative perturbations of \(\hat\gamma_t\).

\begin{figure}[ht]
  \centering

  \begin{subfigure}{0.49\columnwidth}
    \centering
    \includegraphics[width=\textwidth]{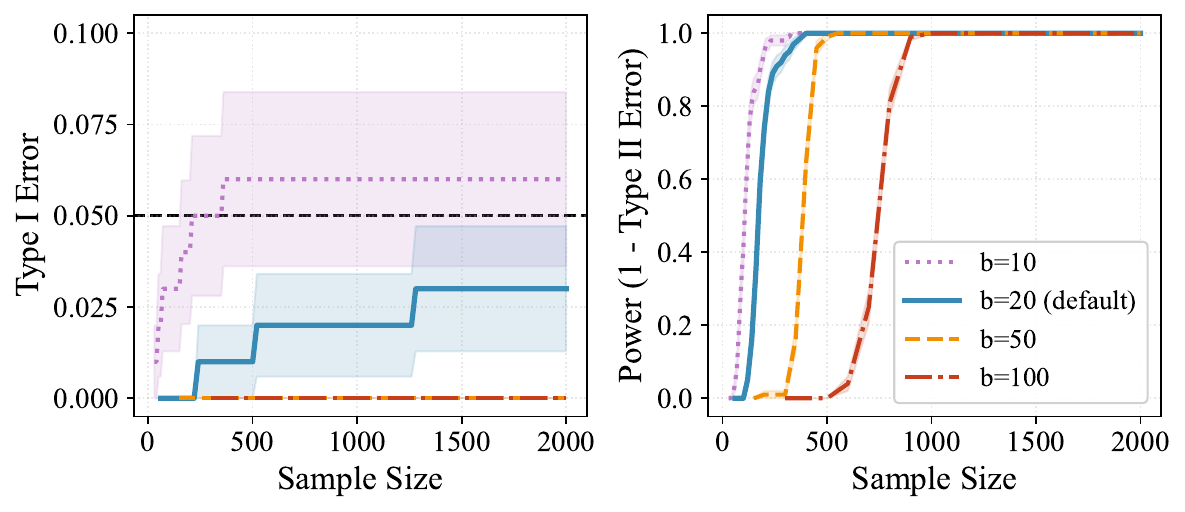}
    \caption{Batch size (1D).}
    \label{ab-sind1-bs}
  \end{subfigure}
  \hfill
  \begin{subfigure}{0.49\columnwidth}
    \centering
    \includegraphics[width=\textwidth]{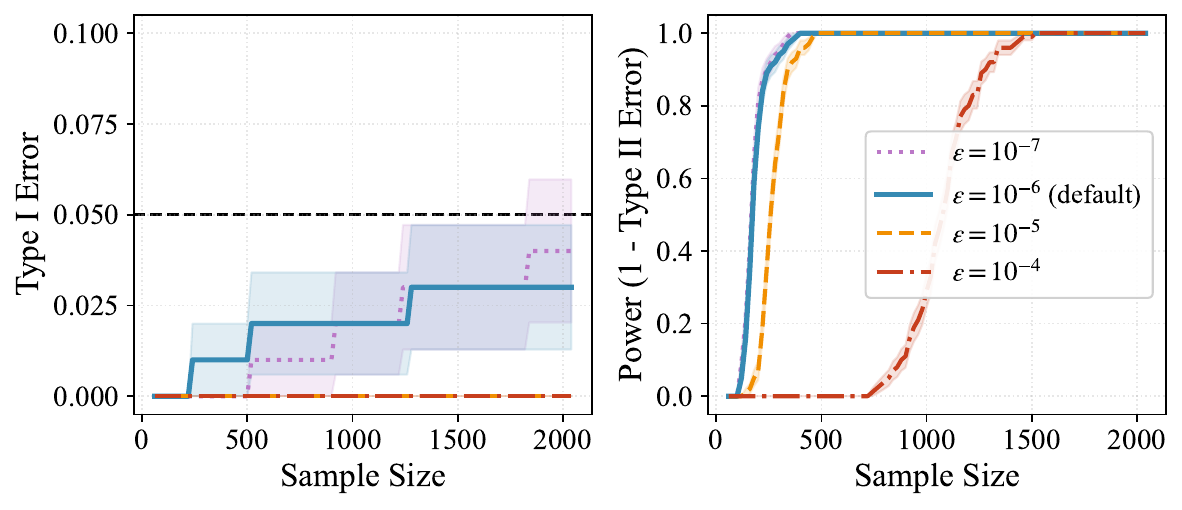}
    \caption{Regularization \(\varepsilon\) (1D).}
    \label{ab-sind1-eps}
  \end{subfigure}

  \vspace{0.5em}

  \begin{subfigure}{0.5\columnwidth}
    \centering
    \includegraphics[width=\textwidth]{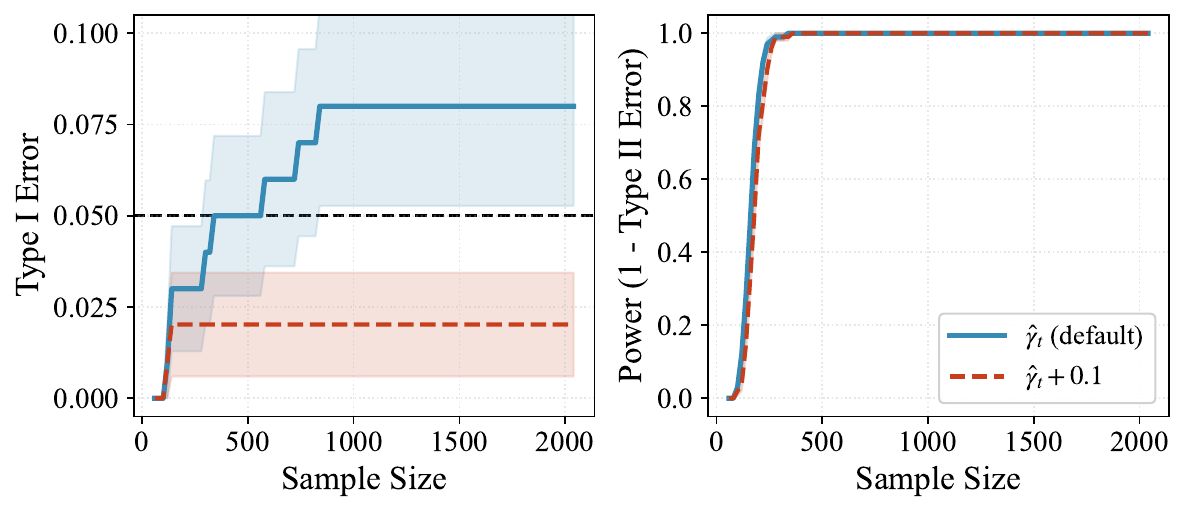}
    \caption{Shift \(\gamma\) (3D, shared coordinate).}
    \label{ab-sind3-gamma}
  \end{subfigure}

  \caption{
  Ablation studies for SKCI on CI-hardness data, Online mode.
  Panels (a) and (b) show the effect of varying the batch size and denominator
  regularization parameter \(\varepsilon\) on the one-dimensional setting.
  Panel (c) studies a conservative perturbation of the shift parameter
  on the three-dimensional shared-coordinate setting, comparing the default
  \(\hat\gamma_t\) with \(\hat\gamma_t + 0.1\).
  }
  \label{fig:skci-ablation}
\end{figure}

\end{document}

%% file: math_commands.tex
\usepackage{amsmath,amsfonts, bm, amsthm, mathtools, amssymb}
\usepackage{centernot}

\DeclareFontFamily{U}{mathx}{\hyphenchar\font45}
\DeclareFontShape{U}{mathx}{m}{n}{<-> mathx10}{}
\DeclareSymbolFont{mathx}{U}{mathx}{m}{n}
\DeclareMathAccent{\widebar}{0}{mathx}{"73}

\usepackage{thmtools}
\usepackage{thm-restate}
\usepackage{dsfont} 
\DeclarePairedDelimiter{\norm}{\lVert}{\rVert}
\DeclarePairedDelimiter{\abs}{\lvert}{\rvert}
\DeclarePairedDelimiterX{\inner}[2]{\langle}{\rangle}{#1,#2}

\def\Hzero{{H_0}}
\def\Hone{{H_1}}
\def\N{{\mathcal{N}}}

\def\A{{\mathcal{A}}}
\def\B{{\mathcal{B}}}
\def\C{{\mathcal{C}}}
\def\H{{\mathcal{H}}}
\def\F{{\mathcal{F}}}
\def\G{{\mathcal{G}}}

\def\Y{{\mathcal{Y}}}
\def\Z{{\mathcal{Z}}}

\def\Xtr{\mathcal{X}^{\mathrm{tr}}}

\def\Xval{\mathcal{X}^{\mathrm{val}}}
\def\tZ{{\widetilde{Z}}}
\def\tA{{\widetilde{A}}}
\def\tV{{\widetilde{V}}}

\def\Vr{V^{\mathrm{raw}}}
\def\hV{{\widehat{V}}}

\def\bmu{{\mu}}
\def\bsigma{{\sigma}}
\DeclareMathOperator{\HS}{HS}

\def\hgamma{\widehat{\gamma}}
\def\hmu{\widehat{\mu}}
\def\hsigma{\widehat{\sigma}}

\def\bE{\widebar{\E}}
\def\bVar{\widebar{\Var}}

\def\muact{\mu_{A\mid C}^{(t)}}
\def\mubct{\mu_{B\mid C}^{(t)}}
\def\detact{{\delta_{A\mid C}^{(t)}}}
\def\detbct{{\delta_{B\mid C}^{(t)}}}

\def\phia{{\phi_A}}
\def\phib{{\phi_B}}
\def\phic{{\phi_C}}

\def\muac{\mu_{A\mid C}}
\def\mubc{\mu_{B\mid C}}

\def\HSnorm2#1{{\norm{#1}^2_{\mathrm{HS}}}}

\newcommand\ind{\protect\mathpalette{\protect\independenT}{\perp}}
\def\independenT#1#2{\mathrel{\rlap{$#1#2$}\mkern2mu{#1#2}}}
\newcommand{\nind}{\centernot{\ind}}

\def\inftynorm2#1{{\norm{#1}^2_\infty}}

\def\1{\bm{1}}

\NewDocumentCommand{\cum}{m m o}{%
  \mathcal{K}^{\IfValueT{#3}{#3}}_{#1}(#2)%
}

\DeclareMathOperator{\E}{\mathbb{E}}

\newcommand{\R}{\mathbb{R}}

\DeclareMathOperator{\Var}{{Var}}

\newcommand{\kvar}[1][]{%
  \kappa_{\mathrm{var}}%
  \if\relax\detokenize{#1}\relax
  \else
    ^{#1}%
  \fi
}

\DeclareMathOperator*{\argmax}{arg\,max}

\newcommand{\defeq}{\coloneq}